\documentclass{article}

\usepackage{ifthen}

\newboolean{anonymous}
\setboolean{anonymous}{false}  

\ifthenelse{\boolean{anonymous}}{}{\PassOptionsToPackage{final,main}{neurips_2026}}

\usepackage{neurips_2026}
\renewcommand{\acksection}{\section*{Acknowledgments}}

\usepackage[utf8]{inputenc}
\usepackage[T1]{fontenc}
\usepackage{hyperref}
\usepackage{url}
\usepackage{booktabs}
\usepackage{amsfonts}
\usepackage{nicefrac}
\usepackage{microtype}
\usepackage{xcolor}

\usepackage{algorithm}
\usepackage{algpseudocode}
\usepackage{amsmath}
\usepackage{amssymb}
\usepackage{amsthm}
\usepackage[titletoc,title]{appendix}
\usepackage{bbm}
\usepackage{booktabs}
\usepackage{cancel}
\usepackage[labelfont=bf]{caption}
\usepackage{enumerate}
\usepackage{graphicx}
\usepackage{multirow}
\usepackage{multicol}
\usepackage{placeins}
\usepackage{romannum}
\usepackage{subcaption}
\usepackage{tabularx}
\usepackage{wrapfig}
\usepackage{xspace}
\usepackage{xcolor,colortbl}

\newlength{\algoboxheight}
\usepackage{siunitx}
\theoremstyle{plain}
\newtheorem{thm}{\protect\theoremname}
\providecommand{\theoremname}{Theorem}

\theoremstyle{plain}

\providecommand{\claimname}{Claim}

\theoremstyle{plain}
\newtheorem{prop}{\protect\propositionname}
\providecommand{\propositionname}{Proposition}

\theoremstyle{plain}

\providecommand{\lemmaname}{Lemma}

\theoremstyle{plain}
\newtheorem{crl}{\protect\corollaryname}[section]
\providecommand{\corollaryname}{Corollary}

\theoremstyle{definition}
\newtheorem{defn}{\protect\definitionname}
\providecommand{\definitionname}{Definition}

\theoremstyle{remark}
\newtheorem{assumption}{Assumption}

\theoremstyle{remark}
\newtheorem{rem}{Remark}

\theoremstyle{plain}
\newtheorem{manualtheoreminner}{Theorem}[thm]
\newenvironment{manualtheorem}[1]{\renewcommand\themanualtheoreminner{#1}\manualtheoreminner
}{\endmanualtheoreminner}

\theoremstyle{plain}

\theoremstyle{plain}

\theoremstyle{plain}

\providecommand{\theoremname}{Theorem}

\theoremstyle{plain}

\providecommand{\lemmaname}{Lemma}

\theoremstyle{plain}

\providecommand{\corollaryname}{Corollary}

\DeclareMathOperator*{\argmax}{arg\,max}

\newcommand{\bydef}{\triangleq}

\newcommand{\CondCurlyBrackets}[1]{\expandafter\ifx\expandafter\relax\detokenize{#1}\relax\else\left(#1\right)\fi}
\newcommand{\CondRectBrackets}[1]{\expandafter\ifx\expandafter\relax\detokenize{#1}\relax\else[#1]\fi}
\newcommand{\CondRectBracketsHigh}[1]{\expandafter\ifx\expandafter\relax\detokenize{#1}\relax\else\left[#1\right]\fi}
\newcommand{\CondColon}[1]{\expandafter\ifx\expandafter\relax\detokenize{#1}\relax\else:{#1}\fi}

\newcommand{\R}{\mathbb{R}}
\newcommand{\N}{\mathbb{N}}
\newcommand{\E}{\mathbb{E}}
\newcommand{\Prob}{\mathbb{P}}
\newcommand{\Sspace}{\mathcal{S}}
\newcommand{\A}{\mathcal{A}}

\newcommand{\abs}[1]{|#1|}
\newcommand{\norm}[1]{\Vert#1\Vert}
\newcommand{\set}[1]{\{#1\}}
\newcommand{\bracks}[1]{[#1]}
\newcommand{\parens}[1]{(#1)}

\newcommand{\givenflat}{\,|\,}
\newcommand{\Ball}{\mathbb{B}}

\newcommand{\dist}{\operatorname{d}}
\newcommand{\argmaxset}{\operatorname*{Argmax}}

\newcommand{\loc}{\mathrm{loc}}

\definecolor{gssactioncolor}{HTML}{1F4E8C}
\definecolor{gssstatecolor}{HTML}{1B9E77}
\definecolor{gsstailcolor}{HTML}{E69F00}
\definecolor{gssscorecolor}{HTML}{D81B60}
\newcommand{\GSSAction}[1]{\textcolor{gssactioncolor}{#1}}
\newcommand{\GSSState}[1]{\textcolor{gssstatecolor}{#1}}
\newcommand{\GSSTail}[1]{\textcolor{gsstailcolor}{#1}}
\newcommand{\GSSScore}[1]{\textcolor{gssscorecolor}{#1}}

\makeatletter
\@ifundefined{linenomath*}{
  \newenvironment{linenomath*}{}{}
}{}
\makeatother
 \ifthenelse{\boolean{anonymous}}{

 }{}

\allowdisplaybreaks[4]
\title{Graph Sparse Sampling: Breaking the Curse\\of the Horizon in Continuous MDP Planning}

\ifthenelse{\boolean{anonymous}}
{
    \author{Anonymous Authors \\
      Anonymous Institution \\
      \texttt{anonymous@anonymous.edu}
    }
}
{
    \author{Idan Lev-Yehudi$^1$ \quad Vadim Indelman$^{2,3}$ \\
      $^1$Technion Autonomous Systems Program (TASP), Technion {--} Israel Institute of Technology\\
      $^2$Stephen B. Klein Faculty of Aerospace Engineering, Technion {--} Israel Institute of Technology\\
      $^3$Faculty of Data and Decision Sciences, Technion {--} Israel Institute of Technology\\
      \texttt{idanlev@campus.technion.ac.il, vadim.indelman@technion.ac.il}
    }
}

\begin{document}

\pagenumbering{arabic}
\maketitle
\begingroup
\makeatletter
\renewcommand{\@makefntext}[1]{\noindent#1}
\footnotetext{Code and data: \url{https://github.com/IdanLevYehudi/GraphSparseSampling}}
\makeatother
\endgroup
\vspace{-0.2em}

\begin{abstract} 
    Planning under uncertainty in continuous domains is essential for autonomous systems, yet computationally demanding. Tree-based search methods such as Monte Carlo Tree Search (MCTS) remain popular, but their branching structure can require sampling budgets that grow exponentially with lookahead depth in the worst case. From a tree perspective, continuous state or action spaces become especially challenging, since the planner must decide where to search in an infinite branching hierarchy. We propose Graph Sparse Sampling (GSS), an online planning algorithm that shares sampled futures across many candidate decisions, rather than sampling separate successors for each candidate action. This branch-free graph exposes large GPU-friendly batches, while using heuristics to focus computation. We prove finite-sample performance guarantees for GSS covering full-rank or low-rank generative simulators via smoothed backups, and discrete or sampled continuous action spaces. Under suitable overlap, regularity, and action-coverage conditions, these bounds have polynomial dependence on the planning horizon, formalizing when shared futures can avoid the exponential horizon dependence of tree-shaped sparse sampling. We demonstrate continuous-control simulations where GSS substantially outperforms tree-based planners on long horizons or achieves near-optimal performance, supporting no-branching graph planning as a complementary design principle for online control.
\end{abstract}

\section{Introduction}

Planning under uncertainty is a fundamental problem in artificial intelligence, and continuous Markov Decision Processes (MDPs) model many physical systems.
For discrete MDPs, exact policy computation is computationally constrained: finite-horizon, discounted, and average-cost variants are polynomial-time solvable but P-complete~\citep{Papadimitriou87math}.
Continuous MDPs add a difficulty, since neither the state space nor the action space can be enumerated directly. 
Sparse Sampling~\citep{Kearns02jml} addresses large state spaces by using only a generative model to sample a sparse set of successor states at each state-action pair, yielding online planning bounds independent of the state dimension.
However, these are exponential in the effective planning horizon, and worst-case tight when only a black-box simulator is available~\citep{Kearns02jml}.

Most practical continuous MDP planners build a search tree, and adapt how the tree is widened or how actions are selected.
Double Progressive Widening (DPW)~\citep{Couetoux11iclio} limits the number of state and action children in continuous Monte Carlo Tree Search; KR-UCT~\citep{Yee16ijcai} shares value estimates between nearby continuous actions;
Voronoi Optimistic Optimization (VOO) and VOOT~\citep{Kim20aaai} use Voronoi-based black-box optimization inside tree search;
and VG-UCT~\citep{Lee20aaai} refines sampled actions using value gradients.
Other non-tree approaches include Model Predictive Path Integral control (MPPI)~\citep{Williams17jgcd}, which optimizes sampled open-loop trajectories with importance-weighted path costs, and stochastic mesh method (SMM)~\citep{Broadie04jcf} and weighted stochastic mesh (WSM)~\citep{Belomestny20mf,Belomestny25joc}, which reuse sampled states in problems such as high-dimensional option pricing and general Markov decision processes.
These methods provide important ways to cope with continuous domains, but they either retain a tree-shaped planning object, optimize trajectories rather than shared Bellman layers, or are not designed around fixed-shape batched online control on GPUs.

We propose Graph Sparse Sampling (GSS), an online planner that replaces separate sampled subtrees with shared successor layers.
Rather than drawing a different successor set for every state-action pair, GSS samples a common next-state layer and evaluates many candidate decisions against it.
We design GSS to trade per-sample adaptivity for sampling throughput, so that GSS can scale to far larger sampling budgets in continuous MDPs.
Our theoretical results formalize this intuition by showing that under suitable conditions, GSS can achieve polynomial error sample complexity in the planning horizon,
while our experiments show it is effective in practice, reaching far higher sample counts than MCTS-based planners.

\subsection{Main Contributions}

\textbf{Algorithmic contribution.} We present Graph Sparse Sampling (GSS), a novel planning algorithm for continuous MDPs designed around state sample reuse in predictable shape operations.
GSS is built for parallel processing on modern GPUs, and incorporates action and state sampling heuristics that focus budgets on more promising regions of the state and action spaces.

\textbf{Theoretical contribution.} We provide a general derivation of finite-time, high-probability guarantees of action-value estimates that hold jointly for the entire sampled planning graph.
Under density-ratio overlap, backup-stability, and action-coverage conditions, comparable to those of previous results, these guarantees depend polynomially on the planning horizon.
Furthermore, we show extensions of the bounds to continuous action spaces, and to low-rank transition models of black-box simulators, under suitable conditions.

\textbf{Empirical contribution.} We demonstrate GSS in three continuous-control benchmarks. We compare it to tree-based planners in similar or extended time budgets, and to closed-loop analytic baselines. GSS is competitive with or improves over tree-based planners in the tested regimes, often at smaller measured planning times, and scales to larger batched sample counts.
Furthermore, we show that GSS scales to high-dimensional and long-horizon settings in a physical simulation.

\subsection{Related Work}

\textbf{Tree-based online planning with guarantees.}
Sparse Sampling~\citep{Kearns02jml} and UCT~\citep{Kocsis06ecml} are classical tree-based baselines with finite-sample or asymptotic guarantees in discrete-action MDPs.
Continuous-domain tree planners add mechanisms for action and state expansion, including DPW~\citep{Couetoux11iclio,Auger13Sp}, kernel regression~\citep{Yee16ijcai}, Voronoi optimization~\citep{Kim20aaai,Lim21cdc}, and value gradients~\citep{Lee20aaai}.
Non-asymptotic theory for MCTS has been developed for finite or discrete-action settings \citep{Shah22or,Barenboim26aij}, and in continuous domains, for tree-based optimistic partitioning methods \citep{Mao2020neurips}. These methods remain tree-shaped and do not analyze shared successor layers or SNIS-style graph backups.

\textbf{Graph-based online planning and sample reuse.}
\citet{Leurent20acml} consider Monte Carlo Graph Search (MCGS), but in discrete MDPs, for repeating state samples.
\citet{Kujanpaa24aamas} considered a graph-based planner for continuous MDPs, but without formal guarantees.
\citet{LevYehudi26ijcai} make reuse of state sample values via IS in a tree setting.
Stochastic Mesh Method (SMM) \citep{Broadie04jcf} and Weighted Stochastic Mesh (WSM) \citep{Belomestny20mf,Belomestny25joc} are the closest methods to GSS in the literature, and \citet{Belomestny25joc} prove finite-time bounds for WSM with polynomial dependence on the horizon.
SMM targets optimal stopping problems, whose decisions are binary, and WSM assumes discrete or discretized action spaces with computable transition densities, so neither method directly applies to the continuous-action, low-rank, or deterministic-transition domains we test in Section~\ref{sec:experiments}.
Its compact-state tractability result relies on global density lower/upper bounds and regularity assumptions, whereas our SNIS analysis is stated in proposal-overlap terms through R{\'e}nyi divergences, which can be considered as more general or not directly comparable. Furthermore, we incorporate search heuristics for action and state sampling, steering the algorithm towards closed-loop online planning on GPU hardware.

\section{Background}

\textbf{Markov Decision Process (MDP):}
we consider MDPs in the form $\langle \mathcal{S},\mathcal{A},p,r,\gamma,L\rangle$.
$\mathcal{S}\subseteq\mathbb{R}^{n_s},\mathcal{A}\subseteq\mathbb{R}^{n_a}$ are the state and action spaces.
We use primes for next-step variables: the transition density $p_{t}(s^{\prime}{\mid} s,a)$ describes the next state after taking action $a\in\mathcal{A}$ at state $s\in\mathcal{S}$ at time $t$.
We use the shorthand $s^{\prime}\sim p_{t}(\cdot\mid s,a)$ to mean that $s^{\prime}$ is sampled from the distribution with this density, and use the same convention for other densities.
The reward function $r_t(s,a,s^\prime)\in\mathbb{R}$ satisfies $\abs{r_t(s,a,s^\prime)}\le R_{\max}$, gives the immediate reward of transitioning from state $s$ to $s^\prime$ by taking action $a$ at time $t$, and the expected reward is $r_t(s,a)\bydef \mathbb{E}_{s^\prime\sim p_t(\cdot\mid s,a)}[r_t(s,a,s^\prime)]$.
$\gamma\in(0,1]$ is the discount factor.
The MDP starts at time $0$ and terminates after $L\in\mathbb{N}\cup\{\infty\}$ steps, so time indices satisfy $0\le t\le L$, and if $L=\infty$ then we assume $\gamma < 1$.
We consider (possibly time-dependent) deterministic policies $\pi=(\pi_t)_{t\in\mathbb{N}}$, i.e. $\pi_t:\Sspace\to\A$, yet we note that our theoretical results can generalize to stochastic policies as well.
The value function of a policy $\pi$ at time $t$ is
$V_{t}^{\pi}(s_{t})\bydef \mathbb{E}^{\pi}[\sum_{i=t}^{L-1}\gamma^{i-t}r_i(s_{i},\pi_{i}(s_{i}),s_{i+1})\mid s_t]$, where $\mathbb{E}^{\pi}[\cdot\mid s_t]$ is taken over $s_{i+1}\sim p_i(\cdot\mid s_i,\pi_i(s_i))$ for $i\ge t$, and we use $m\,{:}\,n$ to denote the range $m,\ldots,n\in\mathbb{Z}$.
For any mapping $U:\Sspace\to\R$, define the single-sample Bellman target
$G_t^U(s,a,s')\bydef r_t(s,a,s')+\gamma U(s')$.
We define the action-value function as
$Q_{t}^{\pi}(s,a)\bydef \mathbb{E}_{s^\prime\sim p_t(\cdot\mid s,a)}[G_t^{V_{t+1}^{\pi}}(s,a,s^\prime)]$.
The optimal value and action-value functions are
$V_t^*(s)\bydef \sup_{\pi}V_t^\pi(s)$ and
$Q_t^*(s,a)\bydef \mathbb{E}_{s^\prime\sim p_t(\cdot\mid s,a)}[G_t^{V_{t+1}^{*}}(s,a,s^\prime)]$ respectively, and we denote $G_t^*(s,a,s') \bydef G_t^{V_{t+1}^*}(s,a,s')$.
We write $\argmaxset_t(s)\bydef\argmax_{a\in\A}Q_t^*(s,a)$ for the exact optimal-action set.
Our planning goal is to return a near-optimal root action at $s_0$.

\textbf{Importance Sampling (IS):} IS~\citep{Kloek78econometrica} is a technique in Monte Carlo (MC) estimation where a proposal distribution $q$ is used to generate samples. We use $\hat{\boxdot}$ for empirical estimators and reserve $\tilde{\boxdot}$ for approximate theoretical quantities. The IS estimator for $g(x)=\mathbb{E}_{x\sim p}[f(x)]$ is $\hat{g}_{q,\mathrm{IS}}=N^{-1} \sum_{i=1}^{N}\rho^{p}_{q}(x^{i})f(x^{i})$, for importance ratios $\rho^{p}_{q}(x) \bydef p(x) \slash q(x)$.
$\hat{g}_{q,\mathrm{IS}}$ is unbiased if $q(x)=0$ implies $p(x)=0$.
Often, the self-normalized IS (SNIS) estimator is formulated when we only have access to an unnormalized version of $p$, or for variance reduction purposes, and is given by $\hat{g}_{q,\mathrm{SNIS}}=(\sum_{i=1}^{N}\rho^{p}_{q}(x^{i}))^{-1} \sum_{i=1}^{N}\rho^{p}_{q}(x^{i})f(x^{i})$.
While biased, under weak assumptions it is consistent~\citep[9.2]{Owen13book}.

\begin{figure}[t]
    \centering
    \includegraphics[width=0.8\columnwidth]{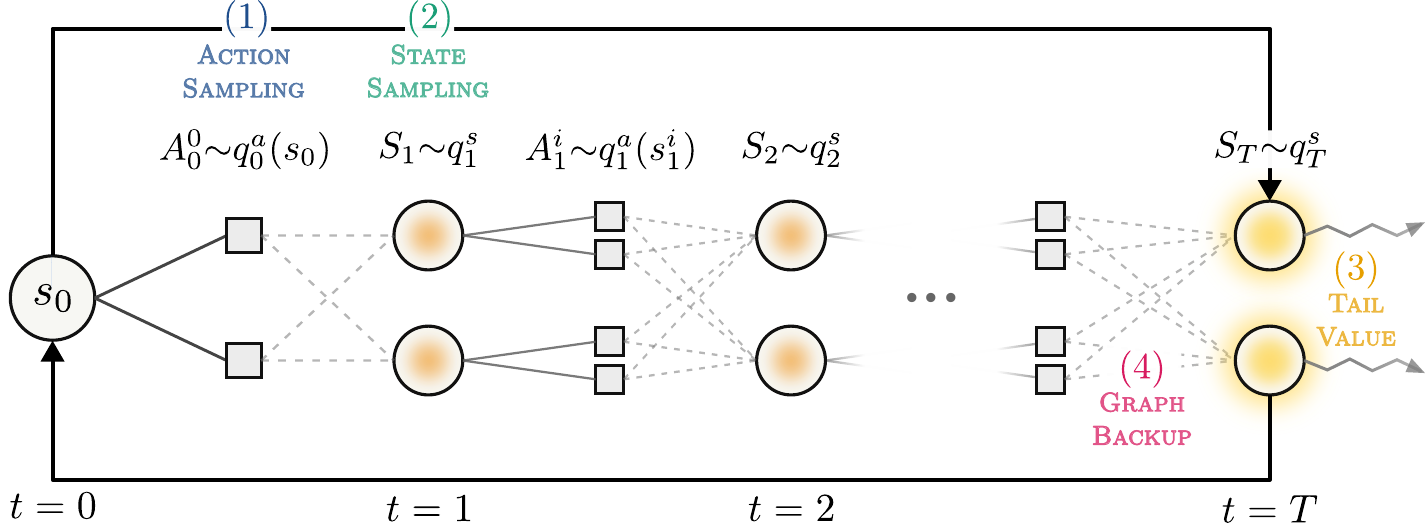}
    \caption{
    GSS builds a layered planning graph from the root state $s_0$.
    Sampled states are shown as circles and actions as rectangles.
    In the forward pass, we draw candidate actions for each state using the (1) \GSSAction{action proposal}; the next shared state layer is then drawn by the (2) \GSSState{state proposal}.
    At depth $T$, nodes are evaluated by a (3) \GSSTail{tail value}.
    The backward pass applies the (4) \GSSScore{graph backup} to each node-action pair, backpropagating the best action-value estimates at each node.
    }
    \label{fig:gss_planning_graph}
\end{figure}

\textbf{Monte Carlo Tree Search (MCTS):} MCTS is an algorithm used to quickly explore large state spaces~\citep{Browne12ieee}. It iteratively repeats four steps to build a search tree that approximates the action-values, using a best-first strategy.
Often UCT~\citep{Kocsis06ecml} is used to balance between exploration of new actions and exploitation of promising ones.
Double Progressive Widening (DPW)~\citep{Couetoux11iclio} is a technique to limit the branching factor from being infinite in continuous settings.
The number of children of a node is artificially limited to $kN^\alpha$ for $N$ visitations, for fixed $k>0$ and $0<\alpha<1$.
VOO samples continuous actions by implicitly using Voronoi cells around evaluated actions to balance local and global search~\citep{Kim20aaai}.
VPW uses VOO when progressive widening adds new actions, extending it to stochastic tree search in continuous or hybrid domains~\citep{Lim21cdc}.

\section{Graph Sparse Sampling (GSS)}
\label{sec:gss}

\subsection{Layered Sparse Graph Structure}

GSS is a finite-lookahead planner from a root state $s_{0}\in\Sspace$, attempting to estimate $\pi_{0}^{*}(s_{0})$.
GSS operates by drawing a known number of state and action samples from time $0$ until its terminal depth $T\le L$.
We design GSS around decoupling successor-state sampling from action-value estimation, separated into forward pass and backward pass procedures respectively.
This decoupling allows GSS to achieve large, regular batches of operations that can be parallelized efficiently on modern GPUs.
Algorithm~\ref{alg:gss} shows pseudocode for GSS, in which each line is a batched operation.

\begin{algorithm}[t]
\caption{Graph Sparse Sampling}
\label{alg:gss}
\begin{algorithmic}[1]
\State Set $s_{0}^{1}\gets s_{0}$ and $S_{0}\gets\{s_{0}^{1}\}$.
\Procedure{EvaluateLayer}{$t,S_{t}$}
    \State \textbf{if} $t=T$, set $\hat{V}_{T}^{j}\gets \GSSTail{\mathrm{TailValue}_{T}}(s_{T}^{j})$ for all $j\in 1{:}C_{T}$ and \textbf{return}
    \State For all $i\in 1{:}C_{t}$, draw $a_{t}^{i,1{:}K_{t}}\sim \GSSAction{q_{t}^{a}}(\cdot\mid s_{t}^{i})$, set $A_{t}\gets\cup_{i\in 1{:}C_{t}}\{a_{t}^{i,1{:}K_{t}}\}$
    \State Fit $q_{t}^{s}\gets \GSSState{\mathrm{FitProposal}_{t}}(t,S_{0:t},A_{0:t})$ on all available information
    \State Draw $s_{t+1}^{1{:}C_{t+1}}\sim q_{t}^{s}$ and set $S_{t+1}\gets\{s_{t+1}^{1{:}C_{t+1}}\}$
    \State \Call{EvaluateLayer}{$t+1,S_{t+1}$}
    \State For all $i\in 1{:}C_t$ and $k\in 1{:}K_t$, set
    $\hat{Q}_{t}^{i,k}\gets \GSSScore{\mathsf{B}_{t}}\parens{s_{t}^{i},a_{t}^{i,k};S_{t+1},\hat{V}_{t+1}^{1{:}C_{t+1}},q_{t}^{s}}$
    \State For all $i\in 1{:}C_{t}$, set $\hat{V}_{t}^{i}\gets \max_{k\in 1{:}K_{t}}\hat{Q}_{t}^{i,k}$
    \State \textbf{if} $t=0$ \textbf{return} $a_{0}^{1,\hat{k}_{0}}$ for $\hat{k}_{0}\in\argmax_{k\in 1{:}K_{0}}\hat{Q}_{0}^{1,k}$
\EndProcedure
\State \Return \Call{EvaluateLayer}{$0,S_{0}$}.
\end{algorithmic}
\end{algorithm}

\textbf{Forward pass.}
The forward pass alternates between action sampling and state sampling.
Starting from the initial state $s_{0}$, let $C_{0:T}\in\N^{T+1}$ be the state layer widths, denoting the $i$'th state layer as $S_{t}=\{s_{t}^{1{:}C_{t}}\}$.
Each state node $s_t$ has an action-sampling budget of 
$K_{t}$ candidate actions, resulting in a total of $C_{t}K_{t}$ candidate actions at depth $t$.

For action sampling, GSS uses an \GSSAction{action proposal} density over $\A$, denoted by $\GSSAction{q_{t}^{a}}(\cdot\mid s)$, in each state $s_{t}^{i}$ to draw $a_{t}^{i,1{:}K_{t}}\overset{\mathrm{i.i.d.}}{\sim}\GSSAction{q_{t}^{a}}(\cdot\mid s_{t}^{i})$.
For state sampling, a \GSSState{state proposal} routine chooses a density $q_{t}^{s}\gets\GSSState{\mathrm{FitProposal}_{t}}(t,S_{0:t},A_{0:t})$ on $\Sspace$, and GSS samples $s_{t+1}^{1{:}C_{t+1}}\overset{\mathrm{i.i.d.}}{\sim}q_{t}^{s}$.
We denote the sets $A_{t}^{i}\bydef\{a_{t}^{i,1{:}K_{t}}\}$, $A_{t}\bydef\cup_{i\in 1{:}C_{t}}A_{t}^{i}$, and $S_{t}=\{s_{t}^{1{:}C_{t}}\}$.
At the terminal layer, GSS initializes
$\hat{V}_{T}^{j} \bydef \GSSTail{\mathrm{TailValue}_{T}}(s_{T}^{j})$
for $j\in 1{:}C_T$.
The components $\GSSAction{q_{t}^{a}}$, $\GSSState{\mathrm{FitProposal}_{t}}$ and $\GSSTail{\mathrm{TailValue}_{T}}$ may be tuned to incorporate domain heuristics, including rollout-based or learned heuristics.

\textbf{Backward pass.}
The backward pass proceeds from $t=T-1$ down to $t=0$, alternating between a \GSSScore{graph backup} and a max-reduction at each layer.
\GSSScore{Graph backup} sets the action-value estimates for each $i\in 1{:}C_t$ and $k\in 1{:}K_t$ via \footnote{While we could backup $\hat{Q}_{t}^{i,k}$ over a sparse subset $E_{t}^{i,k} \subseteq S_{t+1}$ to reduce the time complexity below $O(C_{t} \cdot A_{t} \cdot C_{t+1})$, we find that the dense assignment $E_{t}^{i,k}=S_{t+1}$ is both simpler in theory and runs faster for simple GPU implementations.}$\hat{Q}_{t}^{i,k} \bydef \GSSScore{\mathsf{B}_{t}}\parens{s_{t}^{i},a_{t}^{i,k};S_{t+1},\hat{V}_{t+1}^{1{:}C_{t+1}},q_{t}^{s}}\in\mathbb{R}$, given the next state layer $S_{t+1}$, its value estimates $\hat{V}_{t+1}^{1{:}C_{t+1}}$, and the state proposal density $q_{t}^{s}$.
The node value estimate is the maximum over sampled actions, i.e.
$\hat{V}_{t}^{i} \bydef \max_{k\in 1{:}K_{t}}\hat{Q}_{t}^{i,k}$.
At the root, the returned action is $a_0^{1,\hat{k}_0}$ where $\hat{k}_{0}\in\argmax_{k\in 1{:}K_{0}}\hat{Q}_{0}^{1,k}$, which is then executed by the online planner before the next planning episode. We denote the short-hand Bellman target estimates,
$\hat G_t^{i,k,j}\bydef r_t(s_t^i,a_t^{i,k},s_{t+1}^j)+\gamma \hat V_{t+1}^j$, and give example backups:
\vspace{-1.6em}
\begingroup
\begin{subequations}
\label{eq:gss-backup-instances}
\renewcommand{\theequation}{\theparentequation.\Alph{equation}}
\small
\begin{center}
\begin{minipage}[c]{0.40\linewidth}
\begin{equation}
\label{eq:gss-backup-snis}
\begin{aligned}
\hat{Q}_{t,\mathrm{SNIS}}^{i,k}
=
\eta_{\rho}^{-1} \sum_{j=1}^{C_{t+1}}\rho_{t}^{i,k}(j)\hat G_t^{i,k,j}
\end{aligned}
\end{equation}
\end{minipage}
\begin{minipage}[c]{0.58\linewidth}
\begin{equation}
\label{eq:gss-backup-kde}
\begin{aligned}
\hat{Q}_{t,\mathrm{KDE}}^{i,k}
=
\frac{\sum_{j=1}^{C_{t+1}} w_{t,h}^{i,k}(j)\hat G_t^{i,k,j}}
{\sum_{j=1}^{C_{t+1}} w_{t,h}^{i,k}(j)}
\end{aligned}
\end{equation}
\end{minipage}
\\
\begin{minipage}[c]{0.48\linewidth}
\begin{equation}
\label{eq:gss-backup-nn}
\begin{aligned}
\hat{Q}_{t,\mathrm{NN}}^{i,k}
=
r_{t}(s_{t}^{i},a_{t}^{i,k},\tilde{s}_{t+1}^{i,k})
{}+\gamma \hat V_{t+1}^{j_{\mathrm{NN}}}
\end{aligned}
\end{equation}
\end{minipage}
\begin{minipage}[c]{0.5\linewidth}
\begin{equation}
\label{eq:gss-backup-learned}
\begin{aligned}
\hat{Q}_{t,\mathrm{learn}}^{i,k}
=
\mathsf{F}_{t,\phi}(s_{t}^{i},a_{t}^{i,k};\hat{V}_{t+1}^{1{:}C_{t+1}},\mathcal D_{t})
\end{aligned}
\end{equation}
\end{minipage}
\end{center}
\end{subequations}
\endgroup
For SNIS~\eqref{eq:gss-backup-snis}, $\rho_{t}^{i,k}(j)\bydef p_{t}(s_{t+1}^{j}\mid s_{t}^{i},a_{t}^{i,k})/q_{t}^{s}(s_{t+1}^{j})$ and $\eta_{\rho} = \sum_{j=1}^{C_{t+1}}\rho_{t}^{i,k}(j)$.
When exact transition densities are available, the SNIS backup~\eqref{eq:gss-backup-snis} is a convenient choice for GSS. As we show in Section~\ref{sec:theory}, it allows to derive finite-time performance guarantees with a polynomial error sample complexity in the horizon in Corollary~\ref{crl:gss_sample_complexity}, and it is the basis of the backups in our experiments.

When importance weights are unavailable or impractical to use online, one can use simulator draws in several ways.
KDE approximation~\eqref{eq:gss-backup-kde}~\citep{Wand94book} draws $x_{t+1}^{i,k,1:M}\sim p_{t}(\cdot\mid s_{t}^{i},a_{t}^{i,k})$ and forms $\hat{p}_{t,h}^{i,k}(s)\bydef M^{-1}\sum_{\ell=1}^{M}K_{h}(s-x_{t+1}^{i,k,\ell})$; in~\eqref{eq:gss-backup-kde}, $w_{t,h}^{i,k}(j)\bydef \hat{p}_{t,h}^{i,k}(s_{t+1}^{j})/q_{t}^{s}(s_{t+1}^{j})$ are the resulting SNIS weights.
Nearest-neighbor approximations~\eqref{eq:gss-backup-nn} avoid density estimation completely by drawing $\tilde{s}_{t+1}^{i,k}\sim p_{t}(\cdot\mid s_{t}^{i},a_{t}^{i,k})$ and setting $j_{\mathrm{NN}}\in\operatorname*{argmin}_{j\in 1{:}C_{t+1}}d(\tilde{s}_{t+1}^{i,k},s_{t+1}^{j})$.
In general, any learned approximation may be used as in~\eqref{eq:gss-backup-learned}, conditioning the action-value estimates with offline and online available data $\mathcal D_{t}$.

\subsection{Smoothed SNIS Graph Backups for Low-Rank Generative Simulators}
\label{subsec:gss-low-rank-smoothing}

In many continuous simulators, the transition model is only given by the generative process
\begin{equation}
\label{eq:gss-low-rank-simulator}
s^\prime=f_t(s,a,\xi),
\qquad
\xi\sim p_{\xi,t}(\cdot\mid s,a),
\end{equation}
with $\xi\in\mathbb{R}^{n_\xi}$ and $s,s' \in \mathbb{R}^{n_s}$.
When $n_\xi < n_s$, i.e. the latent noise dimension is smaller than the state dimension, the map $\xi\mapsto f_{t}(s,a,\xi)$ produces a transition model that is generally singular w.r.t. Lebesgue measure, making the importance ratio $p_{t}(\cdot\mid s,a)/q_{t}^{s}$ not well-defined.
Kernel smoothing~\citep{Wand94book} addresses this by convolving the target transition density with a kernel density in $\Sspace$, obtaining a density that is nondegenerate w.r.t. Lebesgue measure:
\begin{equation}
\label{eq:gss-smoothed-density}
p_{t,\tau_{t}}(x\mid s,a)
\bydef
\int k_{t,\tau_{t}}\parens{x\mid f_{t}(s,a,\xi)}\,p_{\xi,t}(d\xi\mid s,a),
\end{equation}
where $k_{t,\tau_{t}}(x\mid y)$ is a smoothing density on $\Sspace$ with bandwidth parameter $\tau_{t}>0$, measurable in $y$, that is locally concentrated around $y$ as $\tau_{t}\to 0$.
We refer to smoothed SNIS backups as those that use the SNIS formula~\eqref{eq:gss-backup-snis} with $p_{t,\tau_{t}}$ in place of $p_{t}$ in the importance ratio.

We can control the incurred bias by tuning the bandwidth $\tau_{t}$, as shown in the following proposition.
\begin{prop}[Smoothing bias]
\label{prop:smoothing-bias}
Assume $G_t^*(s,a,\cdot):\Sspace\to\R$ is measurable for $(s,a)\in\Sspace\times\A$, and $\beta_t^s$-H\"older with constant $L_t^s$, i.e.
$
\abs{G_t^*(s,a,x)-G_t^*(s,a,y)}
\le
L_t^s\dist(x,y)^{\beta_t^s}
$ for all $x,y\in\Sspace$. Then:
\begin{align}
&\abs{\E_{s'\sim p_{t}(\cdot\mid s,a)}\bracks{G_t^*(s,a,s')}
-
\E_{x\sim p_{t,\tau_{t}}(\cdot\mid s,a)}\bracks{G_t^*(s,a,x)}}
\le
L_t^s\,m_{\beta_t^s,t}(\tau_{t}),
\label{eq:smoothing-bias-bound}
\\
&\textnormal{where } \qquad m_{\beta_t^s,t}(\tau_{t}) \bydef \sup_{y}\int_{\Sspace}\dist(x,y)^{\beta_t^s} k_{t,\tau_{t}}(x\mid y)\,dx.
\label{eq:smoothing-moment}
\end{align}
\end{prop}
The proof is in Appendix~\ref{app:proof-smoothing-bias}.
The integral~\eqref{eq:gss-smoothed-density} is often unavailable in closed form, but can be approximated by various methods.
In our implementations, we evaluate the linearized approximation around a chosen noise point $\bar\xi$ (often the mean of $p_{\xi,t}$):
\begin{align}
\label{eq:gss-linearized-smoothed-density}
\tilde{p}_{t,\tau_{t}}^{\mathrm{lin}}(x\mid s,a)
\bydef
\int k_{t,\tau_{t}}(x\mid f_{t}(s,a,\bar\xi)+J_{\xi,t}^{s,a}(\xi-\bar\xi))
\,p_{\xi,t}(d\xi\mid s,a),
\quad
J_{\xi,t}^{s,a}
\bydef
D_\xi f_{t}(s,a,\bar\xi).
\end{align}
When both $p_{\xi,t}\sim\mathcal{N}(\mu_{\xi,t},\Sigma_{\xi,t})$ and $k_{t,\tau_{t}}\sim\mathcal{N}(0,\Sigma_{k,t})$ are Gaussian, this approximation is Gaussian with mean $f_{t}(s,a,\bar\xi)+J_{\xi,t}^{s,a}(\mu_{\xi,t}-\bar\xi)$ and covariance
$J_{\xi,t}^{s,a}\Sigma_{\xi,t}(J_{\xi,t}^{s,a})^\top + \Sigma_{k,t}$, which yields a closed-form density evaluation.
Other approximations such as sigma-point transformations~\citep{Julier04ieee} or Monte Carlo density estimates~\citep{Owen13book} may be used in various settings.

\section{GSS Theoretical Analysis}
\label{sec:theory}

\subsection{Settings and Assumptions}

We first isolate the action-side approximation introduced by replacing the full action space by nodewise sampled action sets.
\begin{defn}[Max-close action sets]
\label{def:max-close-action-sets}
For each depth $t<T$ and realized node $s_t^i$, we say that the action set $A_t^i$ is \emph{max-close} with radius $\Gamma_t\ge0$ and local failure probability $\delta_t^a\in[0,1]$ if
\begin{equation}
\label{eq:action-max-close}
0
\le
V_t^*(s_t^i)-\max_{a\in A_t^i}Q_t^*(s_t^i,a)
\le
\Gamma_t
\end{equation}
except with probability at most $\delta_t^a$.
If $\A$ is finite and $A_t^i=\A$, then \eqref{eq:action-max-close} holds with $\Gamma_t=0$ and $\delta_t^a=0$.
\end{defn}
When $\A$ is metric and the action sets $A_t^i$ are sampled from $\GSSAction{q_t^a}$, the following assumptions yield explicit $\Gamma_t$ and $\delta_t^a$.

\begin{assumption}[Proposal small-ball coverage]
\label{ass:small-ball}
For each depth $t<T$, there is a radius cap $\bar\varepsilon_{t}>0$ such that for every $\varepsilon\in(0,\bar\varepsilon_{t}]$ and every $s\in\Sspace$, there exists $a^{\star}\in\argmaxset_t(s)$ such that
$\GSSAction{q_{t}^{a}} (
\{a\in\A:d_{\A}(a,a^{\star}) <\varepsilon\}
\givenflat s)
\ge
m_{t}(\varepsilon)$
with $m_{t}(\varepsilon)>0$.
\end{assumption}

\begin{assumption}[Action-side H\"older regularity]
\label{ass:holder-action}
For each depth $t<T$, there are constants $L_t^a\ge0$ and $\beta_t^a\in(0,1]$ such that
$\abs{Q_{t}^{*}(s,a)-Q_{t}^{*}(s,a')}
\le
L_t^a\, d_{\A}(a,a')^{\beta_t^a}$
for every $s\in \Sspace$ and all $a,a'\in\A$.
\end{assumption}

\begin{prop}[One-step sampled-action loss]
\label{prop:sampled-action-operator-loss}
Fix $\varepsilon_{t}\in(0,\bar\varepsilon_{t}]$ and a node index $i\in 1{:}C_t$, and suppose that the candidate action set $A_t^i=\set{a_t^{i,1{:}K_t}}$ is sampled conditionally i.i.d.\ from $\GSSAction{q_t^a}(\cdot\mid s_t^i)$.
Under Assumption~\ref{ass:small-ball} and Assumption~\ref{ass:holder-action}, with probability at least $1-\parens{1-m_t(\varepsilon_t)}^{K_t}$, the one-step loss from maximizing only over the sampled actions satisfies
\begin{equation}
\label{eq:sampled-action-one-step-loss}
0
\le
V_t^*(s_t^i)-\max_{a\in A_t^i}Q_t^*(s_t^i,a)
\le
L_t^a\varepsilon_{t}^{\beta_t^a}
\end{equation}
\end{prop}
The proof is in Appendix~\ref{app:proof-sampled-action-operator-loss}.
Under Proposition~\ref{prop:sampled-action-operator-loss}, the sampled action sets are max-close in the sense of Definition~\ref{def:max-close-action-sets}, with $\Gamma_t=L_t^a\varepsilon_t^{\beta_t^a}$ and $\delta_t^a=\parens{1-m_t(\varepsilon_t)}^{K_t}$.

For the following \GSSScore{graph backup} assumptions, we assume that they hold true for all sampled-graph index pairs $i\in 1{:}C_t$ and $k\in 1{:}K_t$.
We write $V_{t+1}^{*,1{:}C_{t+1}}
\bydef \parens{V_{t+1}^{*}(s_{t+1}^{j})}_{j\in 1{:}C_{t+1}}$ for the optimal-value vector at the next layer, and $\hat V_{t+1}^{1{:}C_{t+1}}$ for the realized next-layer value vector.
For the \GSSScore{graph backup} of an arbitrary vector $\mathbf{v}\in\mathbb{R}^{C_{t+1}}$, we shorthand $\mathcal B_t^{i,k}[\mathbf{v}] \bydef \GSSScore{\mathsf{B}_{t}}\parens{s_{t}^{i},a_{t}^{i,k};S_{t+1},\mathbf{v},q_{t}^{s}}$.
For action-value estimates we denote
$Q_t^{*,i,k}\bydef Q_t^*(s_t^{i},a_t^{i,k})$,
$\hat Q_t^{i,k,\star}\bydef \mathcal B_t^{i,k}[V_{t+1}^{*,1{:}C_{t+1}}]$,
and
$\hat Q_t^{i,k}\bydef \mathcal B_t^{i,k}[\hat V_{t+1}^{1{:}C_{t+1}}]$ respectively for the exact optimal action-value at $(s_t^i,a_t^{i,k})$, the graph backup output under exact next-layer optimal values, and the realized recursive backup output.

\begin{assumption}[Graph backup bound]
\label{ass:graph-backup-error-bound} 
For every $i\in 1{:}C_t$, $k\in 1{:}K_t$, and $\lambda_t>0$, the function $\delta_{\loc,t}$ bounds the failure probability:
$
\Prob\!(
\abs{\hat Q_t^{i,k,\star}-Q_t^{*,i,k}}
\le
\lambda_t
)
\ge
1-\delta_{\loc,t}(\lambda_t,C_{t+1}).
$
\end{assumption}

\begin{assumption}[Graph backup stability]
\label{ass:graph-backup-stability}
For every $i\in 1{:}C_t$, $k\in 1{:}K_t$, $\alpha\ge0$, and $\mathbf{v}\in\mathbb R^{C_{t+1}}$, if $\norm{\mathbf{v}-V_{t+1}^{*,1{:}C_{t+1}}}_\infty\le \alpha$, then
$\abs{\mathcal B_t^{i,k}[\mathbf{v}]-\hat Q_t^{i,k,\star}}
\le
\gamma\alpha .$
\end{assumption}
\begin{defn}[Controlled graph backup]
\label{def:controlled-graph-backup}
$\GSSScore{\mathsf{B}_{t}}$ is
\emph{controlled} if it satisfies Assumptions~\ref{ass:graph-backup-error-bound} and~\ref{ass:graph-backup-stability}.
\end{defn}
\noindent
For a controlled \GSSScore{graph backup}, the following bound holds:
\begin{prop}[Controlled graph backup bound]
\label{prop:recursive-graph-backup}
On the event
$\norm{\hat V_{t+1}^{1{:}C_{t+1}}-V_{t+1}^{*,1{:}C_{t+1}}}_\infty\le\alpha_{t+1}$,
under Assumption~\ref{ass:graph-backup-error-bound} and Assumption~\ref{ass:graph-backup-stability}, with probability at least $1-C_tK_t\,\delta_{\loc,t}(\lambda_t,C_{t+1})$, the recursive one-step graph backup error satisfies
\begin{equation}
\label{eq:recursive-graph-backup-error}
\abs{\hat Q_t^{i,k}-Q_t^{*,i,k}}
\le
\gamma\alpha_{t+1}
+\lambda_t
\qquad
\text{for every } i\in 1{:}C_t \text{ and } k\in 1{:}K_t.
\end{equation}
\end{prop}
The proof is in Appendix~\ref{app:graph-backup-details}.
For densities $p$ and $q$, define the exponentiated R{\'e}nyi divergences~\citep{van2014renyi} $d_\infty(p\|q)=\operatorname{ess\,sup}_{x\sim q}\,p(x)/q(x)$ and $d_2(p\|q)\bydef\E_{x\sim q}[(p(x)/q(x))^{2}]$.
For example, under the following theorem's assumptions with a bounded $d_\infty$ divergence, the SNIS backup is a controlled graph backup in the sense of Definition~\ref{def:controlled-graph-backup}, with $\lambda_t=\lambda$ and $\delta_{\loc,t}=3\exp \left(-N\,t(\lambda,N)^2\right)$:
\begin{manualtheorem}{}[SNIS concentration, bounded $d_\infty$ {\citet{Lim23jair}}]
\label{thm:lim-dinf}
Let $d_\infty(p\|q)<\infty$.
Let $x_1,\dots,x_N\overset{i.i.d.}{\sim}q$ and let $f$ satisfy $\|f\|_\infty\le f_{\max}$.
For any $\lambda>0$ such that $\lambda > f_{\max}d_\infty(p\|q)/\sqrt{N}$, define
$t(\lambda,N)\bydef \frac{\lambda}{f_{\max}d_\infty(p\|q)}-\frac{1}{\sqrt{N}}.$
Then
\begin{equation}
\label{eq:lim-dinf-concentration}
\Prob\!\left(\abs{\hat{\mu}_{\mathrm{SNIS},N}-\E_{x\sim p}[f(x)]}\le \lambda\right)
\ge 1-3\exp\!\left(-N\,t(\lambda,N)^2\right).
\end{equation}
\end{manualtheorem}
Appendix~\ref{app:graph-backup-details} shows a controlled median-of-means (MoM) SNIS \GSSScore{graph backup}, for a bounded $d_2$ assumption.

\subsection{Graph Concentration Inequalities and Implications}
\label{subsec:graph-concentration}

Our main graph concentration inequality result turns single-node bounds into graph guarantees.
We first combine the per-depth action and backup parameters into error and failure budgets.
\begin{defn}[Graph error budgets]
\label{def:graph-error-budgets}
Fix a planning depth $T$, a tail accuracy $\kappa_T^{\mathrm{tail}}\ge0$ with per-node tail failure probability $\Delta_T^{\mathrm{tail}}\in[0,1]$, and, for each $t<T$, max-close parameters $\Gamma_t,\delta_t^a$ (Definition~\ref{def:max-close-action-sets}) and backup parameters $\lambda_t,\delta_{\loc,t}$ (Assumption~\ref{ass:graph-backup-error-bound}).
The per-depth error and failure budgets are
$\kappa_t\bydef\lambda_t+\Gamma_t$ and
$\Delta_t\bydef C_tK_t\,\delta_{\loc,t}(\lambda_t,C_{t+1})+C_t\,\delta_t^a$,
and the propagated value radii are $\alpha_T\bydef\kappa_T^{\mathrm{tail}}$ and $\alpha_t\bydef\gamma\alpha_{t+1}+\kappa_t$ for $t=T-1,\dots,0$.
\end{defn}
\begin{thm}[Graph recursive concentration inequality]
\label{thm:master-graph-concentration}
Fix finite realized candidate actions $a_t^{i,1{:}K_t}$ and the budgets of Definition~\ref{def:graph-error-budgets}.
Assume that
(i) for each terminal node $i\in 1{:}C_T$, $\abs{\hat V_T^i-V_T^{*,i}}\le\kappa_T^{\mathrm{tail}}$ except with probability at most $\Delta_T^{\mathrm{tail}}$, where $V_T^{*,i}\bydef V_T^*(s_T^i)$;
(ii) for each $t<T$, the graph backup satisfies the conclusion of Proposition~\ref{prop:recursive-graph-backup} with parameter $\lambda_t$, e.g., because it is controlled (Definition~\ref{def:controlled-graph-backup}); and
(iii) for each $t<T$ and $i\in 1{:}C_t$, the action set $A_t^i$ is max-close with parameters $\Gamma_t$ and $\delta_t^a$.
Then, with probability at least $1-C_T\Delta_T^{\mathrm{tail}}-\sum_{t=0}^{T-1}\Delta_t$, for every $t<T$:
\begin{equation}
\label{eq:master-graph-concentration-bounds}
\norm{\hat Q_t^{1{:}C_t,1{:}K_t}-Q_t^{*,1{:}C_t,1{:}K_t}}_{\infty}
\le
\gamma\alpha_{t+1}+\lambda_t,
\qquad
\norm{\hat V_t^{1{:}C_t}-V_t^{*,1{:}C_t}}_{\infty}\le\alpha_t.
\end{equation}
Moreover, for every $a_0^*\in\argmaxset_0(s_0)$, it holds that
$Q_0^*(s_0,a_0^*)-Q_0^*(s_0,\hat a_0)
\le
2\alpha_0$.
\end{thm}
The proof is given in Appendix~\ref{app:proof-master-graph-concentration}, and the proof strategy is similar to that of \citet{Lim23jair}.

We now show how Propositions~\ref{prop:smoothing-bias} and~\ref{prop:sampled-action-operator-loss} apply to Theorem~\ref{thm:master-graph-concentration}.
\begin{crl}[Sampled-action graph budgets]
\label{crl:sampled-action-graph-concentration}
In Definition~\ref{def:max-close-action-sets}, the fully enumerated case $A_t^i=\A$ gives $\Gamma_t=0$ and $\delta_t^a=0$, thus $\kappa_t=\lambda_t$ and $\Delta_t=C_t\abs{\A}\,\delta_{\loc,t}(\lambda_t,C_{t+1})$. If instead $A_t^i=\set{a_t^{i,1{:}K_t}}$ is sampled as in Proposition~\ref{prop:sampled-action-operator-loss}, then $\Gamma_t=L_t^a\varepsilon_t^{\beta_t^a}$ and $\delta_t^a=\parens{1-m_t(\varepsilon_t)}^{K_t}$, and thus $\kappa_t=\lambda_t+L_t^a\varepsilon_t^{\beta_t^a}$ and $\Delta_t=C_tK_t\,\delta_{\loc,t}(\lambda_t,C_{t+1})+C_t\parens{1-m_t(\varepsilon_t)}^{K_t}$.
\end{crl}

\begin{crl}[Smoothed-transition graph budgets]
\label{crl:smoothed-transition-graph-concentration}
Suppose the exact transition density is replaced by the smoothed density $p_{t,\tau_t}(\cdot\mid s,a)$ and $G_t^*(s,a,\cdot)$ satisfies the assumptions of Proposition~\ref{prop:smoothing-bias}.
Then Proposition~\ref{prop:recursive-graph-backup} holds for the smoothed SNIS backup with the same $\delta_{\loc,t}$ and with $\lambda_t$ replaced by $\lambda_t+L_t^s m_{\beta_t^s,t}(\tau_t)$.
\end{crl}

Finally, we show some of the implications that follow by assuming bounded $d_\infty(p_t(\cdot\mid s,a)\,\|\, q_{t}^{s})$ on GSS concentration inequalities with SNIS backups.
\begin{crl}[Finite-action, absolutely continuous transition sample complexity]
\label{crl:gss_sample_complexity}
Assume $\A$ is finite and fully enumerated, the terminal estimate is exact, and the exact full-rank SNIS backup is used.
Assume also that
$\sup_{t<T,\,s,a}
d_\infty(p_t(\cdot\mid s,a)\,\|\, q_{t}^{s})
\le d_\infty^{\max}$.
Then, for every fixed $\epsilon>0$ and root failure probability $\delta\in(0,1)$, choosing a common layer width
$C
=\widetilde{O}(
R_{\max}^{2}(d_\infty^{\max})^{2}T^{4} \slash \epsilon^{2}
)$
suffices to make the returned root action $\hat a_0$ satisfy
\begin{equation}
\label{eq:gss-root-action-guarantee}
\Pr\!\left(
Q_0^*(s_0,a_0^*)-Q_0^*(s_0,\hat a_0)\le\epsilon
\right)
\ge
1-\delta
\qquad
\text{for every } a_0^*\in\argmaxset_0(s_0).
\end{equation}
Thus, the number of simulator queries is
$N_{\mathrm{GSS}}=TC=\widetilde{O}\!\left(R_{\max}^{2}(d_\infty^{\max})^{2}T^{5}/\epsilon^{2}\right)$, and the number of backup operations is $O(\abs{\A} T C^2)=\widetilde{O}\!\left(\abs{\A} R_{\max}^{4}(d_\infty^{\max})^{4}T^{9}/\epsilon^{4}\right)$.
\end{crl}

\begin{rem}[Comparison with Sparse Sampling]
Corollary~\ref{crl:gss_sample_complexity} gives
$N_{\mathrm{GSS}}=\widetilde{O}(R_{\max}^{2}(d_\infty^{\max})^{2}T^{5}/\epsilon^{2})$.
If $d_\infty^{\max}=O(T^n)$ for $n{\in}\N$, then $N_{\mathrm{GSS}}$ remains polynomial in $T$, whereas the Sparse Sampling budget $N_{\mathrm{SS}}$ required for target accuracy $\epsilon$ is exponential in $T$~\citep{Kearns02jml}.
See Appendix~\ref{app:proof-gss-sample-complexity}, which also gives a simple example with $n=n_s$.
\end{rem}

\begin{crl}[Joint schedule for low-rank sampled-action GSS]
\label{crl:joint-schedule}
Assume the hypotheses of Corollaries~\ref{crl:sampled-action-graph-concentration} and~\ref{crl:smoothed-transition-graph-concentration}.
Suppose, in addition, that the terminal estimate is exact, or that its tail accuracy and failure budgets can be made small enough after the final width $C_T$ is fixed.
Assume also that $L_t^s m_{\beta_t^s,t}(\tau_t)\to0$ as $\tau_t\downarrow0$, and that the smoothed backup can be controlled, uniformly over realized node-action pairs, to any prescribed local accuracy $\lambda_t>0$ and local failure level $\rho_t\in(0,1)$ by increasing the state-side width $C_{t+1}$.
Then for every $\varepsilon_{\mathrm{root}}>0$ and $\delta\in(0,1)$, one can choose bandwidths $\tau_t$, action tolerances $\varepsilon_t$, action budgets $K_t$, and state-side widths $C_{t+1}$ so that $2\alpha_0\le\varepsilon_{\mathrm{root}}$ and $C_T\Delta_T^{\mathrm{tail}}+\sum_{t=0}^{T-1}\Delta_t\le\delta$.
Consequently,
$\Pr\left(
Q_0^*(s_0,a_0^*)-Q_0^*(s_0,\hat a_0)
\le
\varepsilon_{\mathrm{root}}
\right)
\ge
1-\delta$ for every $a_0^*\in\argmaxset_0(s_0)$.
\end{crl}

\section{Experiments and Results}
\label{sec:experiments}

We evaluate GSS in three continuous-control settings chosen to stress different parts of the method: controlled rollout mismatch, singular non-linear dynamics, and smoothed backups in high-dimensional state spaces.
Full configurations, including action mixtures, state proposals, smoothing bandwidths, and DPW/VPW parameters, are in Appendix~\ref{app:experiments}.

\textbf{Implementation.}
GSS is implemented in JAX~\citep{Bradbury18github} with fixed-shape batched graph construction and runs on a GPU.
In Rotating DDI and Lunar Lander, we compare it to the tree-based planners DPW~\citep{Couetoux11iclio} and VPW~\citep{Lim21cdc}, which use the same rollout policy as GSS and run on a CPU with the MCTS.jl solvers of POMDPs.jl~\citep{egorov2017pomdps} in Julia.
A tree search expands one node at a time, each expansion depending on the statistics of the previous ones, which suits a sequential CPU implementation, whereas GSS issues the transitions and backups of each layer as large batched operations, which a GPU executes in parallel.
As a result, GSS draws far more samples per decision in a given planning time: counting every sampled state, including those inside rollouts, GSS draws more than six times as many as DPW and VPW runs with equal or longer planning times (Appendix~\ref{app:sample-counts}).
Our time comparisons are therefore end-to-end comparisons of each method in an implementation suited to its computational structure.

The implementation is guided by the analysis of Section~\ref{sec:theory}.
Every \GSSAction{action proposal} includes a uniform component, consistent with Assumption~\ref{ass:small-ball}, and every \GSSScore{graph backup} is an SNIS backup, exact in Rotating DDI and smoothed in Lunar Lander and Reacher, whose normalized nonnegative weights give the stability of Assumption~\ref{ass:graph-backup-stability}.
The implementation also uses heuristics outside the analyzed algorithm: \GSSState{state proposals} are fitted to the sampled transitions of the current layer and, in Rotating DDI and Lunar Lander, restricted to the components with the highest rollout scores; smoothed densities use the linearization~\eqref{eq:gss-linearized-smoothed-density}; \GSSTail{tail values} are single-sample returns of $\pi_{\mathrm{rollout}}$; and in Rotating DDI and Lunar Lander, the backup includes a direct successor sample.

\textbf{Direct successor sample.}
For each pair $(s_t^i,a_t^{i,k})$, this heuristic draws $\tilde s_{t+1}^{i,k}\sim p_t(\cdot\mid s_t^i,a_t^{i,k})$ and values it by $\hat G_t^{i,k,0}\bydef r_t(s_t^i,a_t^{i,k},\tilde s_{t+1}^{i,k})+\gamma\hat R_{t+1}(\tilde s_{t+1}^{i,k})$, where $\hat R_{t+1}$ is a single-sample return of $\pi_{\mathrm{rollout}}$.
Since $\tilde s_{t+1}^{i,k}$ is drawn from the transition model rather than from $q_t^s$, it enters the backup with unit weight:
\begin{equation}
\label{eq:gss-backup-direct}
\hat Q_{t,\mathrm{DS}}^{i,k}
=
\frac{\hat G_t^{i,k,0}+\sum_{j=1}^{C_{t+1}}\rho_t^{i,k}(j)\,\hat G_t^{i,k,j}}{1+\eta_\rho},
\end{equation}
with $\rho_t^{i,k}$ and $\eta_\rho$ as in the SNIS backup~\eqref{eq:gss-backup-snis}, or its smoothed version in Lunar Lander.
The direct sample has weight $1/(1+\eta_\rho)$, and $\eta_\rho$ has expectation $C_{t+1}$ when $q_t^s$ covers the support of the density in the ratios, $p_t(\cdot\mid s_t^i,a_t^{i,k})$ or its smoothed version.
When few graph states fall where this density has mass, $\eta_\rho$ is small and the estimate approaches the rollout-based target $\hat G_t^{i,k,0}$; when $\eta_\rho$ is near its expectation, the direct sample has weight about $1/(C_{t+1}+1)$.
This parallels MCTS, which initializes the value of a new node with a rollout and refines it as its subtree grows.

\begin{figure*}[t]
    \centering
    \begin{subfigure}[t]{0.31\linewidth}
        \centering
        \includegraphics[width=\linewidth]{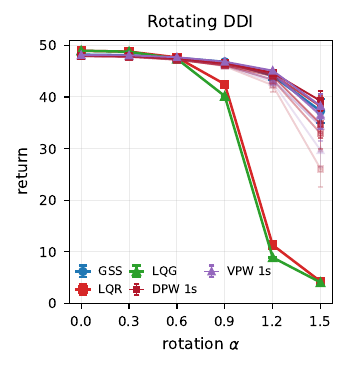}
        \caption{Rotating DDI}
        \label{fig:ddi_results}
    \end{subfigure}
    \hfill
    \begin{subfigure}[t]{0.31\linewidth}
        \centering
        \includegraphics[width=\linewidth]{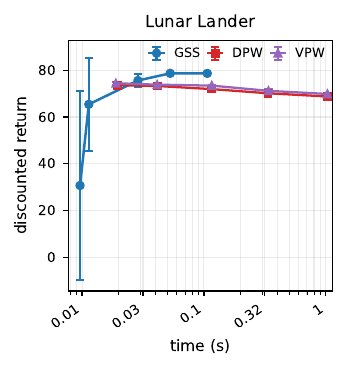}
        \caption{Lunar Lander}
        \label{fig:lunar_results}
    \end{subfigure}
    \hfill
    \begin{subfigure}[t]{0.31\linewidth}
        \centering
        \includegraphics[width=\linewidth]{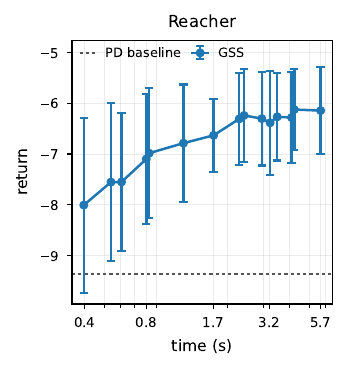}
        \caption{Reacher}
        \label{fig:reacher_results}
    \end{subfigure}
    \caption{
        Error bars show $\pm 2$ standard errors.
        We report mean planner performance for displayed domains, over discounted returns for Lunar Lander, or over undiscounted returns for the other domains.
        In Rotating DDI, the x-axis is the rotation parameter $\alpha$, which controls the mismatch of the rollout guide and the difficulty of the problem.
        DPW/VPW curves span wall-clock budgets $0.01$, $0.0316$, $0.1$, $0.316$, and $1.0$ seconds, with lower-budget curves faded and the $1.0$ second curves emphasized; GSS mean runtimes are at most $0.7$.
        In Lunar Lander and Reacher, the x-axis is mean planner runtime per planning session, averaged over the closed-loop execution in the environment. 
    }
    \label{fig:experiment_results}
\end{figure*}

\textbf{Rotating 4D-Double-Integrator (Rotating DDI).}
This domain evaluates GSS under controlled rollout mismatch, with exact transition densities and SNIS backups.
Extending the benchmark of \citet{Mansley11icaps}, the state is $s=(x,v)\in\mathbb{R}^{8}$ and the acceleration action is $a\in[-2,2]^4$. However, an externally rotating frame introduces
$a_{\alpha}=a-2\alpha Jv+\alpha^2 x$, where $J$ rotates coordinate pairs $(1,2)$ and $(3,4)$ by $90^\circ$.
The transition density is the full-rank Gaussian $p_t(s'\mid s,a)=\mathcal{N}(s';\mu_{\alpha}(s,a),\Sigma)$, where $\mu_{\alpha}(s,a)$ is given by Euler numerical integration over the decision interval $\Delta t=0.2$, with diagonal isotropic covariance $\Sigma=0.02^2 I_8$.
Episodes last at most $50$ steps.

We compare against the clipped LQR feedback baseline
$a_{\mathrm{LQR}}(x,v)=-(x-x_{\mathrm{goal}})-\sqrt{2}\,v$ projected to $[-2,2]^4$, against DPW and VPW with per-decision time budgets from $0.01$ to $1.0$ seconds, and also report a full-state LQG/Riccati reference that is analytically optimal when $\alpha=0$ and actions are unbounded.
Thus $\alpha=0$ is the matched-guide setting: the LQR rollout/proposal follows the nominal double-integrator controller, and the LQG/Riccati reference is analytically optimal.
As the external rotation $\alpha$ increases, the same rollout/proposal becomes increasingly mismatched, so the sweep tests GSS under degrading rollout and proposal quality.
GSS uses an LQR-guided \GSSAction{action proposal}; as \GSSState{state proposal}, the equally weighted mixture of the transition moments of the top-$0.95$ fraction of state-action components, ranked by their direct-sample targets $\hat G_t^{i,k,0}$; the LQR rollout policy for \GSSTail{tail values}; and the exact SNIS \GSSScore{graph backup} with the direct successor sample~\eqref{eq:gss-backup-direct}.
GSS plans with horizon $50$ and rollouts capped at $50$ steps from the current decision, with action budget $64$ and state width $512$; its mean planning time is about $0.2$ seconds per decision at $\alpha=0$ and $0.6$--$0.7$ seconds for $\alpha>0$.

The results are shown in Figure~\ref{fig:ddi_results}.
As $\alpha$ increases, LQR and LQG degrade sharply from $\alpha=0.9$ onward, while GSS maintains relatively consistent performance despite having a degraded rollout guide.
For all tested time budgets, GSS achieves better or comparable performance to DPW/VPW. The tree-based planners show a clear degradation pattern as $\alpha$ increases, achieving competitive performance with GSS in high-rotation regimes only for the highest time budgets.
The full results table is in Appendix~\ref{tab:ddi-budget-sweep}.

\textbf{Lunar Lander.}
This domain evaluates GSS with singular, non-linear dynamics.
The continuous Lunar Lander benchmark of \citet{Mern21aaai} features a 6D spacecraft attempting to land safely from an initially random position.
The state is $(x,z,\theta,v_x,v_z,\omega)\in\mathbb{R}^6$, containing horizontal position, altitude, attitude, and their respective velocities, and the action $a\in\A\subseteq\mathbb{R}^3$ consists of lateral thrust, main thrust, and thruster offset.
The transition model is a deterministic transition function perturbed by 3D additive Gaussian noise, thus it is singular in $\mathbb{R}^6$; its generative equations are given in~\eqref{eq:lunar-transition} in Appendix~\ref{app:exp-lunar}.
Episodes last at most $24$ steps.

We compare against DPW and VPW with per-decision time budgets from $0.01$ to $1.0$ seconds; all three planners use the same brake-near-ground guide and rollout policy.
GSS samples its \GSSAction{action proposal} from multiple sources: the guide's deterministic action, perturbed rollout actions, randomly sampled actions, and a few expert-chosen actions.
Its \GSSState{state proposal} is a diagonal Gaussian fitted to the direct successor samples of the top-$0.8$ fraction of state-action components, ranked by $\hat G_t^{i,k,0}$; it uses the brake-near-ground rollout policy for \GSSTail{tail values}; and its \GSSScore{graph backup} is SNIS with the single-point smoothed density~\eqref{eq:gss-linearized-smoothed-density} and the direct successor sample~\eqref{eq:gss-backup-direct}.
GSS plans with horizon $12$ and rollouts capped at $24$ steps from the current decision; we chose five graph shapes, with action budgets from $10$ to $50$ and state widths from $384$ to $2048$, whose mean planning times are at most about $0.1$ seconds per decision.

The results are shown in Figure~\ref{fig:lunar_results}, where each planner is placed at its measured mean planning time.
GSS reaches the high-discounted-return regime from action budget $24$ and state width $1024$ and remains stable as the action and state budgets grow.
The tree planners are strongest at short plotted budgets, but their plotted mean performances decline as time budgets increase (see Appendix~\ref{app:exp-lunar}).

\textbf{Reacher.}
This domain evaluates smoothed backups in a high-dimensional state space with deterministic transitions.
We implement Reacher-v5~\citep{Towers24arxiv} in JAX~\citep{Bradbury18github} using the MJX interface to MuJoCo~\citep{Todorov12iros}.
The 10D state contains joint-angle cosines and sines, target coordinates, arm angular velocities, and fingertip-target displacement; the actions are two clipped actuator torques.
The transition dynamics are deterministic.
Episodes last at most $50$ steps.

We compare against a deterministic proportional-derivative (PD) controller.
GSS uses this controller in the \GSSAction{action proposal}, with one controller action and the remaining action budget split between noisy-controller and uniform samples.
Its \GSSState{state proposal} is a transition-moment mixture without top-component pruning; it uses the controller as the rollout policy for \GSSTail{tail values}; and its \GSSScore{graph backup} is SNIS with the single-point smoothed density~\eqref{eq:gss-linearized-smoothed-density}.
GSS plans with horizon $20$ and rollouts capped at $50$ steps from the current decision; we vary the state width from $256$ to $2048$ and the action budget from $24$ to $64$, resulting in mean planning times from $0.4$ to $5.7$ seconds per decision.

The results are shown in Figure~\ref{fig:reacher_results}.
Every tested GSS budget improves the local controller, and increasing the state width and action budget generally improves performance with diminishing returns.
This shows that GSS can improve baseline policies and scale effectively in a high-dimensional, deterministic-transition domain.

\section{Conclusion}

We present Graph Sparse Sampling, an alternative to tree-based online planning, by sampling shared successor state layers, evaluating many candidate actions against them, and using generalized graph backups to propagate values through a fixed-shape planning graph.
GSS changes the bottleneck from branching a search tree, to designing effective state and action proposals to focus on relevant portions of the action and state spaces, yet with enough overlap to support reliable reuse.
Our analysis makes this tradeoff explicit, and we show explicit finite-sample performance guarantees in cases of density-ratio overlap, sampled action coverage, and bounded smoothing-bias.
We obtain performance guarantees with polynomial dependence on the planning horizon in regimes where the shared state layers remain representative, and we discuss how to jointly schedule the state and action budgets and smoothing bandwidths to achieve any desired root accuracy and failure probability in low-rank, sampled-action settings.
Empirically, the same design supports large GPU batches, and improves or matches MCTS-based baselines across continuous-control problems with stochastic and deterministic transition cases, challenging horizons, and varying domain dimensions.

\paragraph{Limitations.}
GSS is proposal-limited: poor action or state proposals can make the graph spend most of its computation in the wrong regions.
Our theoretical analysis is also limited in scope.
Theorem~\ref{thm:master-graph-concentration} applies to any controlled graph backup, but we establish this property only for SNIS backups and their smoothed and median-of-means variants; other graph backups, such as the KDE, nearest-neighbor, and learned backups~\eqref{eq:gss-backup-kde}--\eqref{eq:gss-backup-learned}, require their own analysis and may admit different guarantees.
The polynomial dependence on the planning horizon holds under the stated overlap, regularity, and coverage conditions, and weakens when proposals mismatch the transitions so that $d_\infty^{\max}$ is large.
As for our experiments, they do not isolate horizon scaling; they measure the usefulness of GSS in practice, including heuristics beyond the analyzed algorithm.

\paragraph{Future work.}
It is natural to extend GSS into multiple passes, where proposals are adapted based on earlier value estimates, or use learned approximations as part of GSS components.
We plan to explore these directions, along with tightening our performance guarantees, adapting to anytime settings, and analyzing our methods in partially observable domains.

\begin{ack}
This work was supported by the Israel Ministry of Innovation, Science and Technology.
\end{ack}

\bibliographystyle{plainnat}
\bibliography{refs}

\begin{thebibliography}{33}
\providecommand{\natexlab}[1]{#1}
\providecommand{\url}[1]{\texttt{#1}}
\expandafter\ifx\csname urlstyle\endcsname\relax
  \providecommand{\doi}[1]{doi: #1}\else
  \providecommand{\doi}{doi: \begingroup \urlstyle{rm}\Url}\fi

\bibitem[Auger et~al.(2013)Auger, Couetoux, and Teytaud]{Auger13Sp}
David Auger, Adrien Couetoux, and Olivier Teytaud.
\newblock Continuous upper confidence trees with polynomial
  exploration--consistency.
\newblock In \emph{Machine Learning and Knowledge Discovery in Databases:
  European Conference, ECML PKDD 2013, Prague, Czech Republic, September 23-27,
  2013, Proceedings, Part I 13}, pages 194--209. Springer, 2013.

\bibitem[Barenboim and Indelman(2026)]{Barenboim26aij}
M.~Barenboim and V.~Indelman.
\newblock Online {POMDP} planning with anytime deterministic optimality
  guarantees.
\newblock \emph{Artificial Intelligence}, 350:\penalty0 104442, 2026.
\newblock \doi{https://doi.org/10.1016/j.artint.2025.104442}.

\bibitem[Belomestny et~al.(2020)Belomestny, Kaledin, and
  Schoenmakers]{Belomestny20mf}
Denis Belomestny, Maxim Kaledin, and John Schoenmakers.
\newblock Semitractability of optimal stopping problems via a weighted
  stochastic mesh algorithm.
\newblock \emph{Mathematical Finance}, 30\penalty0 (4):\penalty0 1591--1616,
  2020.

\bibitem[Belomestny et~al.(2025)Belomestny, Schoenmakers, and
  Zorina]{Belomestny25joc}
Denis Belomestny, John Schoenmakers, and Veronika Zorina.
\newblock Weighted mesh algorithms for general markov decision processes:
  Convergence and tractability.
\newblock \emph{Journal of Complexity}, 88:\penalty0 101932, 2025.

\bibitem[Bradbury et~al.(2018)Bradbury, Frostig, Hawkins, Johnson, Katariya,
  Leary, Maclaurin, Necula, Paszke, Vander{P}las, Wanderman-{M}ilne, and
  Zhang]{Bradbury18github}
James Bradbury, Roy Frostig, Peter Hawkins, Matthew~James Johnson, Yash
  Katariya, Chris Leary, Dougal Maclaurin, George Necula, Adam Paszke, Jake
  Vander{P}las, Skye Wanderman-{M}ilne, and Qiao Zhang.
\newblock {JAX}: composable transformations of {P}ython+{N}um{P}y programs,
  2018.
\newblock URL \url{http://github.com/jax-ml/jax}.

\bibitem[Broadie et~al.(2004)Broadie, Glasserman, et~al.]{Broadie04jcf}
Mark Broadie, Paul Glasserman, et~al.
\newblock A stochastic mesh method for pricing high-dimensional american
  options.
\newblock \emph{Journal of Computational Finance}, 7:\penalty0 35--72, 2004.

\bibitem[Browne et~al.(2012)Browne, Powley, Whitehouse, Lucas, Cowling,
  Rohlfshagen, Tavener, Perez, Samothrakis, and Colton]{Browne12ieee}
Cameron~B Browne, Edward Powley, Daniel Whitehouse, Simon~M Lucas, Peter~I
  Cowling, Philipp Rohlfshagen, Stephen Tavener, Diego Perez, Spyridon
  Samothrakis, and Simon Colton.
\newblock A survey of {Monte Carlo} tree search methods.
\newblock \emph{IEEE Transactions on Computational Intelligence and AI in
  games}, 4\penalty0 (1):\penalty0 1--43, 2012.

\bibitem[Cou{\"e}toux et~al.(2011)Cou{\"e}toux, Hoock, Sokolovska, Teytaud, and
  Bonnard]{Couetoux11iclio}
Adrien Cou{\"e}toux, Jean-Baptiste Hoock, Nataliya Sokolovska, Olivier Teytaud,
  and Nicolas Bonnard.
\newblock Continuous upper confidence trees.
\newblock In \emph{International conference on learning and intelligent
  optimization}, pages 433--445. Springer, 2011.

\bibitem[Dau(2022)]{Dau22thesis}
Hai~Dang Dau.
\newblock \emph{{Sequential Bayesian Computation}}.
\newblock PhD thesis, {Institut Polytechnique de Paris}, September 2022.
\newblock URL \url{https://theses.hal.science/tel-03848268}.

\bibitem[Egorov et~al.(2017)Egorov, Sunberg, Balaban, Wheeler, Gupta, and
  Kochenderfer]{egorov2017pomdps}
Maxim Egorov, Zachary~N Sunberg, Edward Balaban, Tim~A Wheeler, Jayesh~K Gupta,
  and Mykel~J Kochenderfer.
\newblock {POMDPs.jl}: A framework for sequential decision making under
  uncertainty.
\newblock \emph{The Journal of Machine Learning Research}, 18\penalty0
  (1):\penalty0 831--835, 2017.

\bibitem[Julier et~al.(2004)Julier, Jeffrey, and Uhlmann]{Julier04ieee}
Simon~J. Julier, Jeffrey, and K.~Uhlmann.
\newblock Unscented filtering and nonlinear estimation.
\newblock In \emph{Proceedings of the IEEE}, pages 401--422, 2004.

\bibitem[Kearns et~al.(2002)Kearns, Mansour, and Ng]{Kearns02jml}
Michael Kearns, Yishay Mansour, and Andrew~Y Ng.
\newblock A sparse sampling algorithm for near-optimal planning in large
  {M}arkov decision processes.
\newblock \emph{Machine learning}, 49\penalty0 (2):\penalty0 193--208, 2002.

\bibitem[Kim et~al.(2020)Kim, Lee, Lim, Kaelbling, and
  Lozano-P{\'e}rez]{Kim20aaai}
Beomjoon Kim, Kyungjae Lee, Sungbin Lim, Leslie Kaelbling, and Tom{\'a}s
  Lozano-P{\'e}rez.
\newblock {Monte Carlo} tree search in continuous spaces using {Voronoi}
  optimistic optimization with regret bounds.
\newblock In \emph{AAAI Conf. on Artificial Intelligence}, volume~34, pages
  9916--9924, 2020.

\bibitem[Kloek and Van~Dijk(1978)]{Kloek78econometrica}
Teun Kloek and Herman~K Van~Dijk.
\newblock {Bayesian} estimates of equation system parameters: an application of
  integration by {Monte Carlo}.
\newblock \emph{Econometrica: Journal of the Econometric Society}, pages 1--19,
  1978.

\bibitem[Kocsis and Szepesv{\'a}ri(2006)]{Kocsis06ecml}
Levente Kocsis and Csaba Szepesv{\'a}ri.
\newblock Bandit based {M}onte-{C}arlo planning.
\newblock In \emph{European conference on machine learning}, pages 282--293.
  Springer, 2006.

\bibitem[Kujanp{\"a}{\"a} et~al.(2024)Kujanp{\"a}{\"a}, Babadi, Zhao, Kannala,
  Ilin, and Pajarinen]{Kujanpaa24aamas}
Kalle Kujanp{\"a}{\"a}, Amin Babadi, Yi~Zhao, Juho Kannala, Alexander Ilin, and
  Joni Pajarinen.
\newblock Continuous {Monte Carlo} graph search.
\newblock In \emph{Intl. Conf. on Autonomous Agents and Multiagent Systems
  (AAMAS)}, pages 1047--1056, 2024.

\bibitem[Lee et~al.(2020)Lee, Jeon, Kim, and Kim]{Lee20aaai}
Jongmin Lee, Wonseok Jeon, Geon-Hyeong Kim, and Kee-Eung Kim.
\newblock {Monte-Carlo} tree search in continuous action spaces with value
  gradients.
\newblock In \emph{Proceedings of the AAAI conference on artificial
  intelligence}, volume~34, pages 4561--4568, 2020.

\bibitem[Leurent and Maillard(2020)]{Leurent20acml}
Edouard Leurent and Odalric-Ambrym Maillard.
\newblock {Monte-Carlo} graph search: the value of merging similar states.
\newblock In Sinno~Jialin Pan and Masashi Sugiyama, editors, \emph{Asian
  Conference on Machine Learning (ACML 2020)}, pages 577 -- 592, Bangkok,
  Thailand, November 18-20 2020.

\bibitem[Lev-Yehudi et~al.(2026)Lev-Yehudi, Novitsky, Barenboim, Benchetrit,
  and Indelman]{LevYehudi26ijcai}
Idan Lev-Yehudi, Michael Novitsky, Moran Barenboim, Ron Benchetrit, and Vadim
  Indelman.
\newblock Action-gradient {Monte} {Carlo} tree search for non-parametric
  continuous {(PO)MDPs}.
\newblock In \emph{Intl. Joint Conf. on AI (IJCAI)}, pages 4474--4482, August
  2026.
\newblock \doi{10.24963/ijcai.2026/498}.

\bibitem[Lim et~al.(2021)Lim, Tomlin, and Sunberg]{Lim21cdc}
Michael~H Lim, Claire~J Tomlin, and Zachary~N Sunberg.
\newblock {Voronoi} progressive widening: efficient online solvers for
  continuous state, action, and observation {POMDPs}.
\newblock In \emph{2021 60th IEEE conference on decision and control (CDC)},
  pages 4493--4500. IEEE, 2021.

\bibitem[Lim et~al.(2023)Lim, Becker, Kochenderfer, Tomlin, and
  Sunberg]{Lim23jair}
Michael~H Lim, Tyler~J Becker, Mykel~J Kochenderfer, Claire~J Tomlin, and
  Zachary~N Sunberg.
\newblock Optimality guarantees for particle belief approximation of {POMDPs}.
\newblock \emph{Journal of Artificial Intelligence Research}, 77:\penalty0
  1591--1636, 2023.

\bibitem[Mansley et~al.(2011)Mansley, Weinstein, and Littman]{Mansley11icaps}
Chris Mansley, Ari Weinstein, and Michael Littman.
\newblock Sample-based planning for continuous action markov decision
  processes.
\newblock In \emph{Proceedings of the International Conference on Automated
  Planning and Scheduling}, volume~21, pages 335--338, 2011.

\bibitem[Mao et~al.(2020)Mao, Zhang, Xie, and Basar]{Mao2020neurips}
Weichao Mao, Kaiqing Zhang, Qiaomin Xie, and Tamer Basar.
\newblock Poly-hoot: Monte-carlo planning in continuous space mdps with
  non-asymptotic analysis.
\newblock volume~33, pages 4549--4559, 2020.

\bibitem[Mern et~al.(2021)Mern, Yildiz, Sunberg, Mukerji, and
  Kochenderfer]{Mern21aaai}
John Mern, Anil Yildiz, Zachary Sunberg, Tapan Mukerji, and Mykel~J
  Kochenderfer.
\newblock {Bayesian} optimized {Monte Carlo} planning.
\newblock In \emph{Proceedings of the AAAI Conference on Artificial
  Intelligence}, volume~35, pages 11880--11887, 2021.

\bibitem[Owen(2013)]{Owen13book}
Art~B. Owen.
\newblock \emph{Monte Carlo theory, methods and examples}.
\newblock \url{https://artowen.su.domains/mc/}, 2013.

\bibitem[Papadimitriou and Tsitsiklis(1987)]{Papadimitriou87math}
Christos~H. Papadimitriou and John~N. Tsitsiklis.
\newblock The complexity of {Markov} decision processes.
\newblock \emph{Mathematics of operations research}, 12\penalty0 (3):\penalty0
  441--450, 1987.

\bibitem[Shah et~al.(2022)Shah, Xie, and Xu]{Shah22or}
Devavrat Shah, Qiaomin Xie, and Zhi Xu.
\newblock Nonasymptotic analysis of {Monte Carlo Tree Search}.
\newblock \emph{Operations Research}, 70\penalty0 (6):\penalty0 3234--3260,
  2022.

\bibitem[Todorov et~al.(2012)Todorov, Erez, and Tassa]{Todorov12iros}
Emanuel Todorov, Tom Erez, and Yuval Tassa.
\newblock {MuJoCo}: A physics engine for model-based control.
\newblock In \emph{2012 IEEE/RSJ International Conference on Intelligent Robots
  and Systems}, pages 5026--5033. IEEE, 2012.
\newblock \doi{10.1109/IROS.2012.6386109}.

\bibitem[Towers et~al.(2024)Towers, Kwiatkowski, Terry, Balis, De~Cola, Deleu,
  Goul{\~a}o, Kallinteris, Krimmel, KG, Perez-Vicente, Pierr{\'e}, Schulhoff,
  Tai, Tan, and Younis]{Towers24arxiv}
Mark Towers, Ariel Kwiatkowski, Jordan Terry, John~U. Balis, Gianluca De~Cola,
  Tristan Deleu, Manuel Goul{\~a}o, Andreas Kallinteris, Markus Krimmel, Arjun
  KG, Rodrigo Perez-Vicente, Andrea Pierr{\'e}, Sander Schulhoff, Jun~Jet Tai,
  Hannah Tan, and Omar~G. Younis.
\newblock Gymnasium: A standard interface for reinforcement learning
  environments.
\newblock \emph{arXiv preprint arXiv:2407.17032}, 2024.

\bibitem[Van~Erven and Harremos(2014)]{van2014renyi}
Tim Van~Erven and Peter Harremos.
\newblock R{\'e}nyi divergence and kullback-leibler divergence.
\newblock \emph{IEEE Transactions on Information Theory}, 60\penalty0
  (7):\penalty0 3797--3820, 2014.

\bibitem[Wand and Jones(1994)]{Wand94book}
Matt~P Wand and M~Chris Jones.
\newblock \emph{Kernel smoothing}.
\newblock CRC press, 1994.

\bibitem[Williams et~al.(2017)Williams, Aldrich, and Theodorou]{Williams17jgcd}
Grady Williams, Andrew Aldrich, and Evangelos~A. Theodorou.
\newblock Model predictive path integral control: From theory to parallel
  computation.
\newblock \emph{Journal of Guidance, Control, and Dynamics}, 40\penalty0
  (2):\penalty0 344--357, 2017.
\newblock \doi{10.2514/1.G001921}.

\bibitem[Yee et~al.(2016)Yee, Lis{\'y}, and Bowling]{Yee16ijcai}
Timothy Yee, Viliam Lis{\'y}, and Michael Bowling.
\newblock {Monte Carlo Tree Search} in continuous action spaces with execution
  uncertainty.
\newblock In Subbarao Kambhampati, editor, \emph{IJCAI}, pages 690--697, 2016.

\end{thebibliography}

\clearpage
\raggedbottom
\appendix
\section{Proofs and Derivations}
\label{app:deferred-theory}

\subsection{Proof of Proposition~\ref{prop:smoothing-bias}}
\label{app:proof-smoothing-bias}

\begin{proof}
Fix $(s,a)\in\Sspace\times\A$.
Let $Y$ have transition law $p_t(\cdot\mid s,a)$ and, conditionally on $Y=y$, let $X$ have density
$k_{t,\tau_t}(\cdot\mid y)$.
By the definition of the smoothed density in~\eqref{eq:gss-smoothed-density}, the marginal law of $X$ has density
$p_{t,\tau_t}(\cdot\mid s,a)$.
Therefore,
\[
\begin{aligned}
&\abs{
\E_{s'\sim p_t(\cdot\mid s,a)}
\bracks{G_t^*(s,a,s')}
-
\E_{x\sim p_{t,\tau_t}(\cdot\mid s,a)}
\bracks{G_t^*(s,a,x)}
}
\\
&\quad=
\abs{
\int_{\Sspace}
\int_{\Sspace}
\parens{G_t^*(s,a,y)-G_t^*(s,a,x)}
k_{t,\tau_t}(x\mid y)\,dx\,p_t(dy\mid s,a)
}
\\
&\quad\le
\int_{\Sspace}
\int_{\Sspace}
\abs{G_t^*(s,a,y)-G_t^*(s,a,x)}
k_{t,\tau_t}(x\mid y)\,dx\,p_t(dy\mid s,a)
\\
&\quad\le
L_t^s
\int_{\Sspace}
\int_{\Sspace}
\dist(x,y)^{\beta_t^s}
k_{t,\tau_t}(x\mid y)\,dx\,p_t(dy\mid s,a)
\\
&\quad\le
L_t^s
\sup_y
\int_{\Sspace}
\dist(x,y)^{\beta_t^s}
k_{t,\tau_t}(x\mid y)\,dx
=
L_t^s\,m_{\beta_t^s,t}(\tau_t).
\end{aligned}
\]
The first inequality is the triangle inequality, the second is the assumed successor-state H\"older regularity of
$G_t^*(s,a,\cdot)$, and the last display is exactly the definition of
$m_{\beta_t^s,t}(\tau_t)$.
\end{proof}

\subsection{Proof of Proposition~\ref{prop:sampled-action-operator-loss}}
\label{app:proof-sampled-action-operator-loss}

\begin{proof}
Fix $t<T$, $\varepsilon_t\in(0,\bar\varepsilon_t]$, and a realized node $s_t^i$.
By Assumption~\ref{ass:small-ball} with $\varepsilon=\varepsilon_t$ and $s=s_t^i$, there exists $a^{\star}\in\argmaxset_t(s_t^i)$ whose $\varepsilon_t$-neighborhood has probability at least $m_t(\varepsilon_t)$ under $\GSSAction{q_t^a}(\cdot\mid s_t^i)$.
Conditionally on $s_t^i$, each candidate action therefore lands in this neighborhood with probability at least $m_t(\varepsilon_t)$.
The $K_t$ candidates in $A_t^i=\set{a_t^{i,1{:}K_t}}$ are conditionally independent, hence all of them miss this neighborhood with probability at most
$\parens{1-m_t(\varepsilon_t)}^{K_t}$.

On the complementary event, choose $a_i\in A_t^i$ with
$d_{\A}(a_i,a^{\star})<\varepsilon_t$.
Since $Q_t^*(s_t^i,a^{\star})=V_t^*(s_t^i)$, Assumption~\ref{ass:holder-action} gives
\[
V_t^*(s_t^i)-Q_t^*(s_t^i,a_i)
=
Q_t^*(s_t^i,a^{\star})-Q_t^*(s_t^i,a_i)
\le
L_t^a\varepsilon_t^{\beta_t^a}.
\]
Maximizing over all actions in $A_t^i$ can only improve on $a_i$, so
\[
0
\le
V_t^*(s_t^i)-\max_{a\in A_t^i}Q_t^*(s_t^i,a)
\le
L_t^a\varepsilon_t^{\beta_t^a}.
\]
This is~\eqref{eq:sampled-action-one-step-loss}.
\end{proof}

\subsection{Additional Graph-Backup Details}
\label{app:graph-backup-details}

\begin{proof}[Proof of Proposition~\ref{prop:recursive-graph-backup}]
Work on the event
\[
\norm{\hat V_{t+1}^{1{:}C_{t+1}}-V_{t+1}^{*,1{:}C_{t+1}}}_\infty
\le
\alpha_{t+1}.
\]
For a fixed pair $i\in 1{:}C_t$ and $k\in 1{:}K_t$, Assumption~\ref{ass:graph-backup-stability} gives
\[
\abs{\hat Q_t^{i,k}-\hat Q_t^{i,k,\star}}
=
\abs{\mathcal B_t^{i,k}[\hat V_{t+1}^{1{:}C_{t+1}}]
-
\mathcal B_t^{i,k}[V_{t+1}^{*,1{:}C_{t+1}}]}
\le
\gamma\alpha_{t+1}.
\]
Assumption~\ref{ass:graph-backup-error-bound} gives
\[
\abs{\hat Q_t^{i,k,\star}-Q_t^{*,i,k}}
\le
\lambda_t
\]
except with probability at most $\delta_{\loc,t}(\lambda_t,C_{t+1})$.
On the intersection of these two events, the triangle inequality yields
\[
\abs{\hat Q_t^{i,k}-Q_t^{*,i,k}}
\le
\gamma\alpha_{t+1}+\lambda_t.
\]
Union-bounding the local concentration failures over the $C_tK_t$ realized node-action pairs gives, with probability at least
$1-C_tK_t\,\delta_{\loc,t}(\lambda_t,C_{t+1})$, the pointwise bound
\[
\abs{\hat Q_t^{i,k}-Q_t^{*,i,k}}
\le
\gamma\alpha_{t+1}+\lambda_t
\qquad
\text{for every } i\in 1{:}C_t \text{ and } k\in 1{:}K_t.
\]
Equivalently,
\[
\norm{\hat Q_t^{1{:}C_t,1{:}K_t}-Q_t^{*,1{:}C_t,1{:}K_t}}_\infty
\le
\gamma\alpha_{t+1}+\lambda_t.
\]
\end{proof}

For probability laws $P\ll Q$, write $\omega(y)\bydef dP/dQ(y)$ and define
\[
d_\infty(P\|Q)\bydef \operatorname*{ess\,sup}_{Y\sim Q}\omega(Y),
\qquad
d_2(P\|Q)\bydef\E_{Y\sim Q}\bracks{\omega(Y)^2}.
\]

\begin{defn}[SNIS and MoM-SNIS]
\label{def:mom-snis}
Given samples $Y_{1:N}\sim Q$, weights $\omega_n=\omega(Y_n)$, and values $f(Y_n)$, SNIS returns
\[
\hat{\mu}_{\mathrm{SNIS},N}(f)
\bydef
\frac{\sum_{n=1}^N\omega_n f(Y_n)}{\sum_{n=1}^N\omega_n}.
\]
For MoM-SNIS, partition $\set{1,\dots,N}$ into nonempty blocks $B_1,\dots,B_B$, form
\[
Z_b(f)
\bydef
\frac{\sum_{n\in B_b}\omega_n f(Y_n)}{\sum_{n\in B_b}\omega_n},
\qquad b=1,\dots,B,
\]
and return
\[
\hat{\mu}_{\mathrm{MoM\text{-}SNIS},N}(f)
\bydef
\operatorname{median}\set{Z_1(f),\dots,Z_B(f)}.
\]
\end{defn}

Ordinary SNIS backups satisfy Assumption~\ref{ass:graph-backup-stability} because their normalized weights are nonnegative and sum to one, so changing the child-value vector by at most $\alpha$ changes the continuation part of the backup by at most $\gamma\alpha$.
The same stability argument holds for median-of-blocks normalized backups: each blockwise normalized average is $\gamma$-Lipschitz in child-value sup norm, and the median is nonexpansive in sup norm.

The following source theorem is a bounded-$d_2$ analogue for MoM-SNIS.
It supplies one possible choice of the local function $\delta_{\loc,t}$ in Assumption~\ref{ass:graph-backup-error-bound}.
Denote
\[
\sigma_{\mathrm{CS}}^2\bydef d_2(P\|Q)-1,
\qquad
\sigma_{\mathrm{IS}}^2
\bydef
\E_{x\sim Q}\!\left[
\omega(x)^2
\left(f(x)-\E_{z\sim P}[f(z)]\right)^2
\right].
\]

\begin{manualtheorem}{}[MoM-SNIS one-sided concentration, bounded $d_2$ {\citet{Dau22thesis}}]
\label{thm:dau-prop2}
Fix $\delta\in(0,1)$ and assume $\sigma_{\mathrm{IS}}^2<\infty$.
If
\[
N \ge 8\left(\max\{32\sigma_{\mathrm{CS}}^2,1\}\right)\log(1/\delta),
\]
then the MoM-SNIS estimator of Definition~\ref{def:mom-snis} satisfies
\[
\Prob\!\left(
\hat{\mu}_{\mathrm{MoM\text{-}SNIS},N}(f)-\E_{x\sim P}[f(x)]
\ge
16\,\sigma_{\mathrm{IS}}\sqrt{\frac{2\log(1/\delta)}{N}}
\right)
\le
\delta.
\]
\end{manualtheorem}

Applying Theorem~\ref{thm:dau-prop2} to $f$ and to $-f$ gives the two-sided statement: for a target local failure level $\rho\in(0,1)$, if
\[
N \ge 8\left(\max\{32\sigma_{\mathrm{CS}}^2,1\}\right)\log(2/\rho),
\]
then
\[
\Prob\!\left(
\abs{
\hat{\mu}_{\mathrm{MoM\text{-}SNIS},N}(f)-\E_{x\sim P}[f(x)]
}
\le
16\,\sigma_{\mathrm{IS}}\sqrt{\frac{2\log(2/\rho)}{N}}
\right)
\ge
1-\rho.
\]
If $\norm{f}_\infty\le f_{\max}$, then
$\sigma_{\mathrm{IS}}^2\le 4f_{\max}^2d_2(P\|Q)$.
Indeed, since
$\abs{f(Y)-\E_{x\sim P}[f(x)]}\le 2f_{\max}$,
\[
\sigma_{\mathrm{IS}}^2
=
\E_{Y\sim Q}\!\left[
\omega(Y)^2\parens{f(Y)-\E_{x\sim P}[f(x)]}^{2}
\right]
\le
4f_{\max}^{2}\E_{Y\sim Q}\bracks{\omega(Y)^2}
=
4f_{\max}^{2}d_2(P\|Q).
\]
We now translate this local estimator statement into the graph-backup notation of
Assumption~\ref{ass:graph-backup-error-bound}.
Fix a depth $t<T$ and a realized node-action pair $(s_t^i,a_t^{i,k})$.
Let
\[
P_t^{i,k}\bydef p_t(\cdot\mid s_t^i,a_t^{i,k}),
\qquad
Q_t\bydef q_t^s,
\qquad
N\bydef C_{t+1},
\]
and apply the MoM-SNIS backup to the exact Bellman target
\[
f_t^{i,k}(x)
\bydef
G_t^*(s_t^i,a_t^{i,k},x)
=
r_t(s_t^i,a_t^{i,k},x)+\gamma V_{t+1}^*(x).
\]
Assume that, uniformly over the realized node-action pairs at depth $t$,
\[
\norm{f_t^{i,k}}_\infty\le f_{\max,t},
\qquad
d_2(P_t^{i,k}\|Q_t)\le D_{2,t}<\infty.
\]
Then for every such pair,
\[
\sigma_{\mathrm{CS},t}^{2}
=d_2(P_t^{i,k}\|Q_t)-1
\le D_{2,t}-1,
\qquad
\sigma_{\mathrm{IS},t}^{2}
\le 4f_{\max,t}^{2}D_{2,t}.
\]
Consequently, for any local failure level $\rho_t\in(0,1)$ such that
\[
C_{t+1}
\ge
8\max\{32(D_{2,t}-1),1\}\log(2/\rho_t),
\]
the exact-value MoM-SNIS graph backup satisfies
\[
\Prob\!\left(
\abs{\hat Q_t^{i,k,\star}-Q_t^{*,i,k}}
\le
32 f_{\max,t}
\sqrt{\frac{2D_{2,t}\log(2/\rho_t)}{C_{t+1}}}
\right)
\ge
1-\rho_t.
\]
Thus Assumption~\ref{ass:graph-backup-error-bound} holds for the MoM-SNIS backup at depth $t$ with the explicit choice
\[
\lambda_t
=
32 f_{\max,t}
\sqrt{\frac{2D_{2,t}\log(2/\rho_t)}{C_{t+1}}},
\qquad
\delta_{\loc,t}(\lambda_t,C_{t+1})=\rho_t,
\]
provided the displayed lower bound on $C_{t+1}$ holds.
Equivalently, a uniform depthwise $d_2$ overlap bound supplies the static, exact-next-value concentration term in the same role that the bounded-$d_\infty$ theorem plays in the main text.
Together with the normalized-weight stability argument above, this makes the MoM-SNIS graph backup controlled in the sense of Definition~\ref{def:controlled-graph-backup}, with the propagated recursive bound then given by Proposition~\ref{prop:recursive-graph-backup}.

\subsection{Proof of Theorem~\ref{thm:master-graph-concentration}}
\label{app:proof-master-graph-concentration}

\begin{proof}
Use the main-text shorthand
\[
V_t^{*,i}\bydef V_t^*(s_t^i),
\qquad
Q_t^{*,i,k}\bydef Q_t^*(s_t^i,a_t^{i,k}).
\]
Let $\kappa_t$, $\Delta_t$, and $\alpha_t$ be the budgets of Definition~\ref{def:graph-error-budgets}.
For each depth $t$, define the value event
\[
\mathcal V_t
\bydef
\left\{
\abs{\hat V_t^i-V_t^{*,i}}\le\alpha_t
\quad\text{for every }i\in 1{:}C_t
\right\}.
\]
For $t<T$, define the simultaneous backup event
\[
\mathcal Q_t
\bydef
\left\{
\abs{\hat Q_t^{i,k}-Q_t^{*,i,k}}
\le
\gamma\alpha_{t+1}+\lambda_t
\quad\text{for every }i\in 1{:}C_t,\ k\in 1{:}K_t
\right\},
\]
and the simultaneous action-coverage event
\[
\mathcal A_t
\bydef
\left\{
0
\le
V_t^{*,i}-\max_{k\in 1{:}K_t}Q_t^{*,i,k}
\le
\Gamma_t
\quad\text{for every }i\in 1{:}C_t
\right\}.
\]
Assumptions (i)--(iii) of the theorem give
\[
\begin{aligned}
\Prob(\mathcal V_T^c)
&\le
\sum_{i\in 1{:}C_T}
\Prob\!\left(\abs{\hat V_T^i-V_T^{*,i}}>\kappa_T^{\mathrm{tail}}\right)
\le C_T\Delta_T^{\mathrm{tail}},
\\
\Prob(\mathcal Q_t^c\cap\mathcal V_{t+1})
&\le
C_tK_t\,\delta_{\loc,t}(\lambda_t,C_{t+1}),
\\
\Prob(\mathcal A_t^c)&\le C_t\,\delta_t^a.
\end{aligned}
\]
The first inequality is assumption (i) with a union bound over terminal nodes.
The second inequality is assumption (ii), i.e., the conclusion of Proposition~\ref{prop:recursive-graph-backup} on the event $\mathcal V_{t+1}$.
The third inequality is the union bound over the $C_t$ max-close action-set failures of assumption (iii).

We next prove the deterministic induction step.
Assume $\mathcal V_{t+1}\cap\mathcal Q_t\cap\mathcal A_t$ holds.
Set
\[
R_t^Q\bydef\gamma\alpha_{t+1}+\lambda_t,
\qquad
R_t\bydef R_t^Q+\Gamma_t
=
\gamma\alpha_{t+1}+\kappa_t
=
\alpha_t.
\]
For every node index $i\in 1{:}C_t$,
\[
\hat V_t^i
=
\max_{k\in 1{:}K_t}\hat Q_t^{i,k}
\le
\max_{k\in 1{:}K_t}Q_t^{*,i,k}+R_t^Q
\le
V_t^{*,i}+R_t.
\]
For the lower deviation,
\[
\begin{aligned}
V_t^{*,i}-\hat V_t^i
&\le
V_t^{*,i}
-
\max_{k\in 1{:}K_t}Q_t^{*,i,k}
+
\max_{k\in 1{:}K_t}Q_t^{*,i,k}
-
\max_{k\in 1{:}K_t}\hat Q_t^{i,k}
\\
&\le
\Gamma_t
+
\max_{k\in 1{:}K_t}\abs{\hat Q_t^{i,k}-Q_t^{*,i,k}}
\le
\Gamma_t+R_t^Q
=
R_t.
\end{aligned}
\]
Thus $\mathcal V_t$ holds.
Equivalently,
\[
\norm{\hat V_t^{1{:}C_t}-V_t^{*,1{:}C_t}}_\infty\le\alpha_t.
\]
The event $\mathcal Q_t$ itself is exactly the stated pointwise candidate-backup bound at depth $t$.
Equivalently,
\[
\norm{\hat Q_t^{1{:}C_t,1{:}K_t}-Q_t^{*,1{:}C_t,1{:}K_t}}_\infty
\le
\gamma\alpha_{t+1}+\lambda_t.
\]

Now define the total bad event
\[
\mathcal F
\bydef
\mathcal V_T^c
\cup
\bigcup_{t=0}^{T-1}\parens{\mathcal Q_t^c\cap\mathcal V_{t+1}}
\cup
\bigcup_{t=0}^{T-1}\mathcal A_t^c .
\]
By the union bound and Definition~\ref{def:graph-error-budgets},
\[
\Prob(\mathcal F)
\le
C_T\Delta_T^{\mathrm{tail}}
+
\sum_{t=0}^{T-1}
\bracks{
C_tK_t\,\delta_{\loc,t}(\lambda_t,C_{t+1})
+C_t\,\delta_t^a
}
=
C_T\Delta_T^{\mathrm{tail}}
+\sum_{t=0}^{T-1}\Delta_t.
\]
On $\mathcal F^c$, the terminal event $\mathcal V_T$ holds.
If $\mathcal V_{t+1}$ holds, then $\mathcal Q_t$ must hold because
$\mathcal Q_t^c\cap\mathcal V_{t+1}$ has been excluded, and $\mathcal A_t$ also holds.
The deterministic induction step therefore gives $\mathcal V_t$.
Backward induction proves that every $\mathcal V_t$ and every $\mathcal Q_t$ hold simultaneously on $\mathcal F^c$.
Thus, for every $t<T$, the pointwise value and candidate-backup inequalities hold for every $i\in 1{:}C_t$ and $k\in 1{:}K_t$.
Equivalently, the two displayed max-norm inequalities in the theorem hold with probability at least
$1-C_T\Delta_T^{\mathrm{tail}}-\sum_{t=0}^{T-1}\Delta_t$.

It remains to prove the root-action consequence on the same event.
Let $\hat k_0\in\argmax_{k\in 1{:}K_0}\hat Q_0^{1,k}$, so that
$\hat a_0=a_0^{1,\hat k_0}$ and $\hat V_0^1=\hat Q_0^{1,\hat k_0}$.
For any $a_0^*\in\argmaxset_0(s_0)$,
\[
\begin{aligned}
Q_0^*(s_0,a_0^*)-Q_0^*(s_0,\hat a_0)
&=
V_0^{*,1}-\hat V_0^1
+
\hat Q_0^{1,\hat k_0}-Q_0^{*,1,\hat k_0}
\\
&\le
\alpha_0
+
\gamma\alpha_1+\lambda_0
\le
2\alpha_0,
\end{aligned}
\]
because $\gamma\alpha_1+\lambda_0\le\gamma\alpha_1+\lambda_0+\Gamma_0=\alpha_0$.
\end{proof}

\subsection{Proof of Corollary~\ref{crl:sampled-action-graph-concentration}}
\label{app:proof-sampled-action-graph-concentration}

\begin{proof}
If $\mathcal A$ is finite and $A_t^i=\mathcal A$ for every node, then
\[
\max_{a\in A_t^i}Q_t^*(s_t^i,a)
=
\max_{a\in\mathcal A}Q_t^*(s_t^i,a)
=
V_t^*(s_t^i)
\]
deterministically.
Thus Definition~\ref{def:max-close-action-sets} holds with
$\Gamma_t=0$ and $\delta_t^a=0$.
Substituting these parameters and $K_t=\abs{\mathcal A}$ into Theorem~\ref{thm:master-graph-concentration} gives
$\kappa_t=\lambda_t$ and
$\Delta_t=C_t\abs{\mathcal A}\,\delta_{\loc,t}(\lambda_t,C_{t+1})$.

For sampled action sets, Proposition~\ref{prop:sampled-action-operator-loss} states that each nodewise action set is max-close with
\[
\Gamma_t=L_t^a\varepsilon_t^{\beta_t^a},
\qquad
\delta_t^a=\parens{1-m_t(\varepsilon_t)}^{K_t}.
\]
Substituting these parameters into Theorem~\ref{thm:master-graph-concentration} gives
\[
\kappa_t
=
\lambda_t+L_t^a\varepsilon_t^{\beta_t^a},
\qquad
\Delta_t
=
C_tK_t\,\delta_{\loc,t}(\lambda_t,C_{t+1})
+
C_t\parens{1-m_t(\varepsilon_t)}^{K_t}.
\]
\end{proof}

\subsection{Proof of Corollary~\ref{crl:smoothed-transition-graph-concentration}}
\label{app:proof-smoothed-transition-graph-concentration}

\begin{proof}
Fix a depth $t<T$ and a realized node-action pair $(s_t^i,a_t^{i,k})$.
For this proof, the local concentration function $\delta_{\loc,t}$ is applied to the SNIS estimator whose target density is the smoothed density
$p_{t,\tau_t}(\cdot\mid s_t^i,a_t^{i,k})$ and whose proposal is
$q_{t}^{s}$.
Thus, with the exact next-layer optimal values inserted into the smoothed SNIS backup, the local concentration event is
\[
\abs{
\hat Q_{t,\tau_t}^{i,k,\star}
-
\E_{x\sim p_{t,\tau_t}(\cdot\mid s_t^i,a_t^{i,k})}
\bracks{G_t^*(s_t^i,a_t^{i,k},x)}
}
\le
\lambda_t,
\]
except with probability at most $\delta_{\loc,t}(\lambda_t,C_{t+1})$.
Proposition~\ref{prop:smoothing-bias} gives the deterministic comparison
\[
\abs{
\E_{x\sim p_{t,\tau_t}(\cdot\mid s_t^i,a_t^{i,k})}
\bracks{G_t^*(s_t^i,a_t^{i,k},x)}
-
Q_t^*(s_t^i,a_t^{i,k})
}
\le
L_t^s m_{\beta_t^s,t}(\tau_t).
\]
Combining the two displays yields the static graph-backup error
\[
\abs{\hat Q_{t,\tau_t}^{i,k,\star}-Q_t^{*,i,k}}
\le
\lambda_t+L_t^s m_{\beta_t^s,t}(\tau_t)
\]
with the same local failure probability.
Assumption~\ref{ass:graph-backup-stability} then adds the usual propagated term
$\gamma\alpha_{t+1}$ when
$\norm{\hat V_{t+1}^{1{:}C_{t+1}}-V_{t+1}^{*,1{:}C_{t+1}}}_\infty\le\alpha_{t+1}$.
The proof of Proposition~\ref{prop:recursive-graph-backup}, with
$\lambda_t$ replaced by
$\lambda_t+L_t^s m_{\beta_t^s,t}(\tau_t)$, gives the corollary.
\end{proof}

\subsection{Proof of Corollary~\ref{crl:gss_sample_complexity}}
\label{app:proof-gss-sample-complexity}

\begin{proof}
In the finite-action fully enumerated case, Corollary~\ref{crl:sampled-action-graph-concentration} gives
$\Gamma_t=0$, $\delta_t^a=0$, and $K_t=\abs{\A}$.
With an exact terminal estimate, $\kappa_T^{\mathrm{tail}}=\Delta_T^{\mathrm{tail}}=0$.
Set $C_0=1$ and choose a common successor-layer width $C_t=C$ for $1\le t\le T$ and a common local radius $\lambda_t=\lambda$.
The recursion in Definition~\ref{def:graph-error-budgets} gives
\[
\alpha_0
=
\sum_{t=0}^{T-1}\gamma^t\lambda
\le
T\lambda,
\]
because $\gamma\le1$ in the finite-horizon setting.
Thus $\lambda=\epsilon/(2T)$ ensures $2\alpha_0\le\epsilon$.

We next choose $C$ so that the graph failure probability is at most $\delta$.
The bounded reward assumption implies, for every $t<T$,
\[
\norm{V_{t+1}^*}_\infty
\le
(T-t-1)R_{\max}.
\]
Therefore the single-sample Bellman target satisfies
\[
\norm{G_t^*(s,a,\cdot)}_\infty
\le
R_{\max}+\gamma (T-t-1)R_{\max}
\le
TR_{\max},
\qquad
B_T\bydef TR_{\max}.
\]
By the uniform overlap assumption,
$d_\infty(p_t(\cdot\mid s,a)\,\|\,q_{t}^{s})\le d_\infty^{\max}$.
Applying Theorem~\ref{thm:lim-dinf} with
$f_{\max}=B_T$, $N=C$, and $d_\infty(p\|q)\le d_\infty^{\max}$ gives
\[
\delta_{\loc,t}(\lambda,C)
\le
3\exp\!\left(
-C
\left(
\frac{\lambda}{B_Td_\infty^{\max}}
-\frac{1}{\sqrt C}
\right)^2
\right),
\]
provided
$\lambda>B_Td_\infty^{\max}/\sqrt C$.
Define
\[
A_\lambda
\bydef
\left(
\frac{B_Td_\infty^{\max}}{\lambda}
\right)^2,
\qquad
\ell_\lambda
\bydef
\max\!\left\{
1,
\log\!\left(
\frac{12eT\abs{\A}A_\lambda}{\delta}
\right)
\right\}.
\]
Choose
\[
C
\ge
8A_\lambda\ell_\lambda .
\]
Then $C\ge4A_\lambda$, so the positivity condition in Theorem~\ref{thm:lim-dinf} holds and
\[
\frac{\lambda}{B_Td_\infty^{\max}}
-\frac{1}{\sqrt C}
=
\frac{1}{\sqrt{A_\lambda}}
-\frac{1}{\sqrt C}
\ge
\frac{1}{2\sqrt{A_\lambda}}.
\]
Consequently, uniformly over depths and realized node-action pairs,
\[
\delta_{\loc,t}(\lambda,C)
\le
3\exp\!\left(-\frac{C}{4A_\lambda}\right).
\]

It remains to check the graph-level union bound.
Since $C_0=1$ and $C_t=C$ for $t\ge1$,
$\sum_{t=0}^{T-1}C_t\le TC$ for $C\ge1$.
Theorem~\ref{thm:master-graph-concentration} and Corollary~\ref{crl:sampled-action-graph-concentration} therefore require
\[
\sum_{t=0}^{T-1}
C_t\abs{\A}\,\delta_{\loc,t}(\lambda,C)
\le
3TC\abs{\A}
\exp\!\left(-\frac{C}{4A_\lambda}\right).
\]
Let $x=C/(4A_\lambda)$ and
$M=12eT\abs{\A}A_\lambda/\delta$.
The choice of $C$ gives $x\ge2\max\{1,\log M\}$.
For such $x$, we have $x e^{-x}\le e/M$:
if $M\le e$, then $x\ge2$ gives $x e^{-x}\le2e^{-2}\le e/M$, while if $M>e$, then $x e^{-x}$ is decreasing for $x\ge1$ and
$2\log(M)/M^2\le e/M$.
Hence
\[
3TC\abs{\A}\exp(-C/(4A_\lambda))
=
12T\abs{\A}A_\lambda x e^{-x}
\le
\delta .
\]
Thus the total failure probability in Theorem~\ref{thm:master-graph-concentration} is at most $\delta$.

Substituting $\lambda=\epsilon/(2T)$ and $B_T=TR_{\max}$ gives
\[
A_\lambda
=
\frac{4R_{\max}^2(d_\infty^{\max})^2T^4}{\epsilon^2}.
\]
Hence the explicit width choice above is
\[
C
\ge
\frac{32R_{\max}^2(d_\infty^{\max})^2T^4}{\epsilon^2}
\max\!\left\{
1,
\log\!\left(
\frac{48eT\abs{\A}R_{\max}^2(d_\infty^{\max})^2T^4}{\delta\epsilon^2}
\right)
\right\}.
\]
Equivalently,
\[
C
=
\widetilde{O}\!\left(
\frac{R_{\max}^2(d_\infty^{\max})^2T^4}{\epsilon^2}
\right),
\]
where the logarithm is the one displayed above.
Theorem~\ref{thm:master-graph-concentration} then gives the stated root-action bound.
Since the graph uses $T$ shared successor layers of width $C$, the simulator-query count scales as $N_{\mathrm{GSS}}=TC$.
For the exact dense SNIS backup, count one backup edge operation for one parent-action-child contribution in the sum~\eqref{eq:gss-backup-snis}.
At depth $t$, there are $C_t\abs{\A}$ enumerated parent-action pairs, and each pair is evaluated against all $C_{t+1}$ child states.
Hence the total number of backup operations is
\[
\sum_{t=0}^{T-1} C_t\abs{\A}C_{t+1}
=
\abs{\A}\parens{C+(T-1)C^2}
\le
\abs{\A}TC^2,
\]
where $C_0=1$, $C_t=C$ for $1\le t\le T$, and $C\ge1$.
Substituting the width bound gives
\[
O(\abs{\A}TC^2)
=
\widetilde{O}\!\left(
\frac{\abs{\A}R_{\max}^{4}(d_\infty^{\max})^{4}T^9}{\epsilon^4}
\right).
\]
Thus, if $d_\infty^{\max}=O(T^n)$ for fixed $n\in\N$, $N_{\mathrm{GSS}}$ scales polynomially in $T$.
\end{proof}

\paragraph{Example of polynomial overlap growth.}
The proof above uses the overlap bound only at realized node-action pairs.
As a deliberately simple example in which $d_\infty^{\max}$ grows polynomially in $T$, suppose that each transition moves the state by at most $\ell>0$, so that $p_t(\cdot\mid s,a)$ is supported in the ball $\Ball(s,\ell)\subseteq\R^{n_s}$, and that the transition densities are bounded, $p_t(x\mid s,a)\le\bar p$ for all $t,s,a,x$.
Let each state proposal be uniform on the ball of states reachable from the root, $q_t^s\bydef\mathrm{Unif}\parens{\Ball(s_0,(t+1)\ell)}$.
Then every node at depth $t$ lies in $\Ball(s_0,t\ell)$, the support of $p_t(\cdot\mid s_t^i,a)$ lies in $\Ball(s_0,(t+1)\ell)$, and for every $t<T$,
\[
d_\infty\parens{p_t(\cdot\mid s_t^i,a)\,\|\,q_t^s}
\le
\bar p\,\mathrm{Vol}\parens{\Ball(s_0,(t+1)\ell)}
=
\bar p\,\omega_{n_s}\parens{(t+1)\ell}^{n_s}
\le
\bar p\,\omega_{n_s}\ell^{n_s}\,T^{n_s},
\]
where $\omega_{n_s}$ is the volume of the unit ball in $\R^{n_s}$.
Hence $d_\infty^{\max}=O(T^{n_s})$, and the width and query bounds above become $C=\widetilde{O}(T^{4+2n_s}/\epsilon^2)$ and $N_{\mathrm{GSS}}=\widetilde{O}(T^{5+2n_s}/\epsilon^2)$, which are polynomial in $T$ for a fixed state dimension.
This proposal ignores all problem structure and differs from the proposals used in our experiments.

\noindent
We use the following planner-to-policy conversion, adapted to our notation from \citet[Lemma 5]{Kearns02jml}.
\begin{manualtheorem}{A.1}[Planner-to-policy conversion]
\label{thm:planner-to-policy-conversion}
Let a planning subroutine be invoked at states encountered by the induced policy.
For the discounted infinite-horizon setting, suppose that for every state $s$, the returned action $\hat a(s)$ satisfies
\[
\Prob\!\left(
Q^*(s,a^*)-Q^*(s,\hat a(s))\le\beta
\right)
\ge
1-\rho,
\qquad
a^*\in\argmax_{a\in\mathcal A}Q^*(s,a).
\]
If $\norm{V^*}_\infty\le V_{\max}$, then the induced randomized policy $\pi$ satisfies
\[
V^*(s)-V^\pi(s)
\le
\frac{\beta+2\rho V_{\max}}{1-\gamma}
\qquad
\text{for every }s.
\]
For a finite horizon of length $T$, if the same condition holds at every state and time with $Q_t^*$ in place of $Q^*$ and $\norm{Q_t^*}_\infty\le TR_{\max}$, then the induced time-dependent policy satisfies
\[
V_0^*(s)-V_0^\pi(s)
\le
T\parens{\beta+2\rho T R_{\max}} .
\]
\end{manualtheorem}

\paragraph{Additional policy-value scalings.}
The main text states the root-action consequence because it follows directly from Theorem~\ref{thm:master-graph-concentration}.
For completeness, we record how the same local-width calculation changes when the planner is invoked repeatedly to induce a policy.
The finite-horizon part of Theorem~\ref{thm:planner-to-policy-conversion} applies once the planner has a per-call action-value gap bound $\beta$ and per-call failure probability $\rho$.
Let $\bar\beta\bydef\beta+2\rho T R_{\max}$.
The theorem gives the policy-value loss bound
\[
V_0^*(s)-V_0^\pi(s)
\le
T\bar\beta
=
T\parens{\beta+2\rho T R_{\max}} .
\]
Here Theorem~\ref{thm:master-graph-concentration} gives $\beta=2\alpha_0\le2T\lambda$ for one planner call.
To make the accumulated deterministic gap at most $\epsilon/2$, choose
$2T\lambda\le\epsilon/(2T)$, equivalently $\lambda=\Theta(\epsilon/T^2)$.
To make the accumulated failure contribution at most $\epsilon/2$, take
$\rho=O(\epsilon/(R_{\max}T^2))$.
Substituting this $\lambda$ into the explicit width calculation above, with $\delta$ replaced by the per-call budget $\rho$, yields
\[
C
=
\widetilde{O}\!\left(
\frac{R_{\max}^2(d_\infty^{\max})^2T^6}{\epsilon^2}
\right),
\qquad
N_{\mathrm{GSS}}
=
\widetilde{O}\!\left(
\frac{R_{\max}^2(d_\infty^{\max})^2T^7}{\epsilon^2}
\right).
\]

For discounted infinite-horizon policy value, let
$V_{\max}=R_{\max}/(1-\gamma)$ and truncate the discounted problem at
$T=\widetilde{O}(1/(1-\gamma))$ so that the tail satisfies
$\gamma^T V_{\max}=O(\epsilon(1-\gamma))$.
The local envelope becomes
$B_\gamma\le R_{\max}+\gamma V_{\max}=V_{\max}$.
Choosing
$\lambda=\Theta(\epsilon(1-\gamma)^2)$ gives
$\alpha_0\le\lambda/(1-\gamma)=O(\epsilon(1-\gamma))$.
The discounted part of Theorem~\ref{thm:planner-to-policy-conversion} then requires a per-call failure budget
$\rho=O(\epsilon(1-\gamma)/V_{\max})$ so that
$(2\alpha_0+2\gamma^T V_{\max}+2\rho V_{\max})/(1-\gamma)=O(\epsilon)$.
Substituting $B_\gamma=O(R_{\max}/(1-\gamma))$ and
$\lambda=\Theta(\epsilon(1-\gamma)^2)$ into
$A_\lambda=(B_\gamma d_\infty^{\max}/\lambda)^2$ gives
\[
C
=
\widetilde{O}\!\left(
\frac{R_{\max}^2(d_\infty^{\max})^2}{\epsilon^2(1-\gamma)^6}
\right).
\]
Multiplying by the truncation depth
$T=\widetilde{O}(1/(1-\gamma))$ yields
\[
N_{\mathrm{GSS}}
=
\widetilde{O}\!\left(
\frac{R_{\max}^2(d_\infty^{\max})^2}{\epsilon^2(1-\gamma)^7}
\right).
\]

\subsection{Proof of Corollary~\ref{crl:joint-schedule}}
\label{app:proof-joint-schedule}

\begin{proof}
Choose strictly positive deterministic budgets $\eta_t^{\mathrm{est}}$, $\eta_t^{\mathrm{moll}}$, and $\eta_t^{\mathrm{act}}$ so that the recursion
\[
\bar\alpha_T=\kappa_T^{\mathrm{tail}},
\qquad
\bar\alpha_t
=
\gamma\bar\alpha_{t+1}
+
\eta_t^{\mathrm{est}}
+
\eta_t^{\mathrm{moll}}
+
\eta_t^{\mathrm{act}},
\qquad
t=T-1,\dots,0,
\]
satisfies $2\bar\alpha_0\le\varepsilon_{\mathrm{root}}$.
Equivalently,
\[
\bar\alpha_0
=
\gamma^T\kappa_T^{\mathrm{tail}}
+
\sum_{t=0}^{T-1}
\gamma^t
\left(
\eta_t^{\mathrm{est}}
+
\eta_t^{\mathrm{moll}}
+
\eta_t^{\mathrm{act}}
\right)
\le
\varepsilon_{\mathrm{root}}/2.
\]
Such budgets exist whenever the terminal estimate is exact, or more generally once the terminal tail accuracy has been made small enough to leave positive slack.
Choose positive failure allocations $\zeta_T^{\mathrm{tail}}$, $\zeta_t^{\mathrm{est}}$, and $\zeta_t^{\mathrm{act}}$ such that
\[
\zeta_T^{\mathrm{tail}}
+
\sum_{t=0}^{T-1}
\parens{\zeta_t^{\mathrm{est}}+\zeta_t^{\mathrm{act}}}
\le
\delta.
\]
In the non-exact terminal case, the terminal estimator is understood to be chosen after the final width $C_T$ is fixed so that
$C_T\Delta_T^{\mathrm{tail}}\le\zeta_T^{\mathrm{tail}}$ for the selected $\kappa_T^{\mathrm{tail}}$.
In the exact terminal case, take $\kappa_T^{\mathrm{tail}}=\Delta_T^{\mathrm{tail}}=0$.

Construct the schedule forward in $t$, starting from $C_0=1$.
At depth $t<T$, assume $C_t$ has already been fixed.
First choose $\tau_t$ so that
\[
L_t^s m_{\beta_t^s,t}(\tau_t)\le\eta_t^{\mathrm{moll}},
\]
which is possible by the assumed vanishing of the smoothing bias.
Next choose $\varepsilon_t\in(0,\bar\varepsilon_t]$ so that
\[
L_t^a\varepsilon_t^{\beta_t^a}\le\eta_t^{\mathrm{act}}.
\]
For instance, when $L_t^a>0$ it is enough to take
$\varepsilon_t\le\min\{\bar\varepsilon_t,(\eta_t^{\mathrm{act}}/L_t^a)^{1/\beta_t^a}\}$; when $L_t^a=0$, any $\varepsilon_t\in(0,\bar\varepsilon_t]$ suffices.
Since Assumption~\ref{ass:small-ball} gives $m_t(\varepsilon_t)>0$, choose
\[
K_t
\ge
\left\lceil
\frac{\log(C_t/\zeta_t^{\mathrm{act}})}
     {m_t(\varepsilon_t)}
\right\rceil .
\]
Then, using $1-x\le e^{-x}$ for $x\ge0$,
\[
\begin{aligned}
K_t m_t(\varepsilon_t)
&\ge
\log(C_t/\zeta_t^{\mathrm{act}}),
\\
C_t\parens{1-m_t(\varepsilon_t)}^{K_t}
&\le
C_t\exp\!\left(-K_t m_t(\varepsilon_t)\right)
\le
C_t\exp\!\left(-\log(C_t/\zeta_t^{\mathrm{act}})\right)
=
\zeta_t^{\mathrm{act}}.
\end{aligned}
\]
Set
\[
\rho_t^{\mathrm{est}}
\bydef
\frac{\zeta_t^{\mathrm{est}}}{C_tK_t}.
\]
Finally, use the assumed uniform controllability of the smoothed backup to choose $C_{t+1}$ and the local SNIS radius $\lambda_t$ so that
\[
\lambda_t\le\eta_t^{\mathrm{est}},
\qquad
\delta_{\loc,t}(\lambda_t,C_{t+1})
\le
\rho_t^{\mathrm{est}},
\]
or equivalently
$C_tK_t\,\delta_{\loc,t}(\lambda_t,C_{t+1})\le\zeta_t^{\mathrm{est}}$.
The same construction can use a common nonroot state width under a bounded-$d_\infty$ condition for the fixed smoothed targets.
Suppose that, uniformly over realized node-action pairs,
\[
d_\infty(p_{t,\tau_t}(\cdot\mid s_t^i,a_t^{i,k})\|q_t^s)\le D_{\infty,t}
\]
and that the corresponding Bellman targets are bounded by $f_{\max,t}$.
Choose fixed local accuracy levels $0<\lambda_t\le\eta_t^{\mathrm{est}}$ and, for a candidate common width $C$, define
\[
C_t(C)
\bydef
\begin{cases}
1, & t=0,\\
C, & t\ge1,
\end{cases}
\qquad
K_t(C)
\bydef
\left\lceil
\frac{\log(C_t(C)/\zeta_t^{\mathrm{act}})}
     {m_t(\varepsilon_t)}
\right\rceil,
\qquad
\rho_t^{\mathrm{est}}(C)
\bydef
\frac{\zeta_t^{\mathrm{est}}}{C_t(C)K_t(C)}.
\]
Let
\[
A_t
\bydef
\left(
\frac{f_{\max,t}D_{\infty,t}}{\lambda_t}
\right)^2.
\]
Choose a common integer $C$ satisfying
\[
C
\ge
\max_{t<T}
4A_t
\log\!\left(\frac{3}{\rho_t^{\mathrm{est}}(C)}\right).
\]
Such a $C$ exists because $K_t(C)$ grows logarithmically in $C$, so the right-hand side grows only logarithmically in $C$, while the left-hand side grows linearly.
Set the actual widths and action budgets by $C_t=C_t(C)$ and $K_t=K_t(C)$.
Since $\rho_t^{\mathrm{est}}(C)<1$, the same lower bound gives $C\ge4A_t$ and hence
\[
\frac{\lambda_t}{f_{\max,t}D_{\infty,t}}
-
\frac{1}{\sqrt C}
=
\frac{1}{\sqrt{A_t}}
-
\frac{1}{\sqrt C}
\ge
\frac{1}{2\sqrt{A_t}}.
\]
Thus the positivity condition in the bounded-$d_\infty$ SNIS concentration bound~\eqref{eq:lim-dinf-concentration} holds, and for $C_{t+1}=C$,
\[
\delta_{\loc,t}(\lambda_t,C)
\le
3\exp\!\left(
-C
\left(
\frac{\lambda_t}{f_{\max,t}D_{\infty,t}}
-
\frac{1}{\sqrt{C}}
\right)^2
\right)
\]
Therefore
\[
\delta_{\loc,t}(\lambda_t,C)
\le
3\exp\!\left(-\frac{C}{4A_t}\right)
\le
\rho_t^{\mathrm{est}}(C),
\]
and hence
$C_t(C)K_t(C)\,\delta_{\loc,t}(\lambda_t,C)\le\zeta_t^{\mathrm{est}}$
at every depth.

With these choices, Corollaries~\ref{crl:sampled-action-graph-concentration} and~\ref{crl:smoothed-transition-graph-concentration} give
\[
\kappa_t
\le
\eta_t^{\mathrm{est}}
+
\eta_t^{\mathrm{moll}}
+
\eta_t^{\mathrm{act}},
\qquad
\Delta_t
\le
\zeta_t^{\mathrm{est}}
+
\zeta_t^{\mathrm{act}}.
\]
Therefore the recursion defining $\alpha_t$ in Definition~\ref{def:graph-error-budgets} gives
$\alpha_t\le\bar\alpha_t$ for every depth, while the total failure probability is at most $\delta$.
The root-action conclusion of Theorem~\ref{thm:master-graph-concentration} then yields
\[
Q_0^*(s_0,a_0^*)-Q_0^*(s_0,\hat a_0)
\le
2\alpha_0
\le
2\bar\alpha_0
\le
\varepsilon_{\mathrm{root}}
\]
with probability at least $1-\delta$, for every $a_0^*\in\argmaxset_0(s_0)$.
\end{proof}

\section{Experimental Details}
\label{app:experiments}

All experiments use the fixed-shape JAX GSS planner described in Section~\ref{sec:experiments}.
The JAX GSS experiments were run on a GPU machine equipped with an NVIDIA GeForce RTX 2080 Ti with 11 GB VRAM.
The Julia DPW/VPW baselines were run on a CPU machine with an AMD Ryzen 7 3700X 8-core/16-thread processor at 3.60 GHz.
We record action-mixture counts in the order primitive/rollout/noisy/uniform.
The primitive component is a domain-specific list of fixed actions, the rollout component repeats the guide-policy action, the noisy component adds Gaussian perturbations scaled by the action range, and the uniform component samples from the action box.

\subsection{GSS Configuration}

We use three state-proposal settings across the experiments.
In Rotating DDI, parent-action components are scored by a sampled one-step transition target plus rollout tail, the top fraction is retained, and successor states are sampled from the retained transition-moment mixture.
Lunar Lander uses the same top-fraction scoring, but the next-state proposal is one diagonal Gaussian fitted to the sampled one-step successors.
Reacher uses an unpruned transition-moment mixture, without top-fraction pruning.
In Rotating DDI and Lunar Lander, the SNIS backup also includes the direct successor sample~\eqref{eq:gss-backup-direct}, whose sampled transition and rollout tail also provide the scores used for top-fraction pruning.

\GSSScore{Graph backups} use either an exact-moment density, when the target transition law is the domain's diagonal Gaussian transition, or a single-point density, when the target transition law is the single-point linearization of the smoothed transition density obtained by adding a positive diagonal bandwidth $\tau$ to the simulator transition moments.
All GSS runs use rollout tail values; rollout tails use the same rollout policy as the action proposal.

\subsection{DPW and VPW Baselines}

The DPW and VPW baselines use the double progressive widening solver of MCTS.jl with parameters
\[
(c_{\mathrm{uct}}, k_a, \alpha_a, k_s, \alpha_s)
=
(60.0, 1.45, 0.60, 0.1, 0.30).
\]
The rollout value estimator simulates the domain's rollout policy under the domain transition.
For DPW, the first action added at a node is the guide-policy action, and later widened actions are sampled uniformly from the action space.
For VPW, the first action is chosen in the same way, but later widened actions are generated by VOO.
In the DDI VPW runs, VOO uses global exploration probability $0.9$ and a diagonal Gaussian local sampler with standard deviation $0.5$ in each action dimension.
In the Lunar VPW implementation, the local covariance is $\operatorname{diag}(0.2^2,0.5^2,0.05^2)$.

\subsection{Sample Counts}
\label{app:sample-counts}

We count the samples each planner draws per decision: every state drawn from the transition model or from a state proposal, including the states visited by rollouts.
GSS draws $C_t$ states from the state proposal at each layer $t\ge1$.
For each state-action pair, it also draws a direct successor sample and rolls it out, and it rolls out each leaf state.
Its rollouts run as fixed-length batches, so we count each rollout at its full capped length.
DPW and VPW draw at most one new successor per search iteration and then roll out to the search depth, which is the episode cap ($35$ steps in Lunar Lander and $50$ in Rotating DDI).
Their runs record only the number of iterations, so we report the number of iterations times the search depth, an upper bound on their samples.
Table~\ref{tab:lunar-planner-performance} reports the Lunar Lander counts and Table~\ref{tab:ddi-sample-counts} the Rotating DDI counts.
Even against DPW and VPW runs with equal or longer planning times, GSS draws more than six times as many samples per decision.

\begin{table}[H]
\centering
\small
\setlength{\tabcolsep}{5pt}
\caption{Rotating DDI samples per decision for each rotation $\alpha$. For GSS, the budget is the action budget $K$ and state width $C$; for DPW and VPW, it is the per-decision time budget. The GSS count does not depend on $\alpha$ and is averaged over the decisions of a 50-step episode, during which its planning horizon shrinks with the remaining steps; DPW and VPW counts are upper bounds.}
\label{tab:ddi-sample-counts}
\begin{tabular}{llcccccc}
\toprule
& & \multicolumn{6}{c}{$\alpha$} \\
\cmidrule(lr){3-8}
Planner & Budget & 0.0 & 0.3 & 0.6 & 0.9 & 1.2 & 1.5 \\
\midrule
GSS & $K{=}64$, $C{=}512$ & 13.66M & 13.66M & 13.66M & 13.66M & 13.66M & 13.66M \\
\midrule
DPW & 0.01s & $\le$0.02M & $\le$0.01M & $\le$0.01M & $\le$0.01M & $\le$0.01M & $\le$0.07M \\
DPW & 0.0316s & $\le$0.04M & $\le$0.03M & $\le$0.03M & $\le$0.03M & $\le$0.03M & $\le$0.12M \\
DPW & 0.1s & $\le$0.13M & $\le$0.09M & $\le$0.09M & $\le$0.09M & $\le$0.10M & $\le$0.28M \\
DPW & 0.316s & $\le$0.35M & $\le$0.25M & $\le$0.26M & $\le$0.25M & $\le$0.30M & $\le$0.59M \\
DPW & 1.0s & $\le$0.90M & $\le$0.72M & $\le$0.72M & $\le$0.72M & $\le$0.83M & $\le$1.46M \\
\midrule
VPW & 0.01s & $\le$0.02M & $\le$0.01M & $\le$0.01M & $\le$0.01M & $\le$0.01M & $\le$0.05M \\
VPW & 0.0316s & $\le$0.04M & $\le$0.03M & $\le$0.03M & $\le$0.03M & $\le$0.03M & $\le$0.09M \\
VPW & 0.1s & $\le$0.12M & $\le$0.09M & $\le$0.09M & $\le$0.09M & $\le$0.10M & $\le$0.21M \\
VPW & 0.316s & $\le$0.34M & $\le$0.25M & $\le$0.25M & $\le$0.25M & $\le$0.29M & $\le$0.56M \\
VPW & 1.0s & $\le$0.88M & $\le$0.71M & $\le$0.71M & $\le$0.71M & $\le$0.80M & $\le$1.35M \\
\bottomrule
 \end{tabular}
\end{table}

\subsection{Rotating 4D Double-Integrator}
\label{app:exp-ddi}

The DDI benchmark uses the 4D double-integrator domain with $D=4$.
States are ordered as $s=(x,v)$ with $x,v\in\mathbb{R}^{4}$, and actions are accelerations $a\in[-2,2]^4$.
The initial state is $(\mathbf{1}_4,\mathbf{0}_4)$ and the goal is the origin.
The rotating-frame experiments use
\[
\Delta t=0.2,\qquad
\sigma_x=\sigma_v=0.02,\qquad
\gamma=0.99.
\]
We also set the drag substep count to 32 and enable the rotating-frame coupling.
The implementation uses one Euler step when drag and nonlinear coupling are inactive, and 32 Euler substeps when the rotating-frame coupling is active; in this sweep, $\alpha=0$ is the single-step nominal DDI case and nonzero $\alpha$ uses the 32-substep rotating-frame integration.
Within each microstep,
\[
x^{+}=x+\Delta t_{\mathrm{sub}} v,\qquad
v^{+}=v+\Delta t_{\mathrm{sub}} a_{\alpha},
\]
where
\[
a_{\alpha}=a-2\alpha Jv+\alpha^2 x
\]
and $J$ rotates each adjacent coordinate pair by $90^\circ$.
The sampled transition then adds independent Gaussian noise with standard deviation $0.02$ to each position and velocity coordinate.
Episodes terminate when any coordinate leaves $[-2,2]$ or after 50 steps.
The reward uses the current state and clipped action,
\[
r(s,a,s')=
1-\frac{\lVert x\rVert_2^2+\lVert a\rVert_2^2}{8D},
\]
with no velocity cost.

The rotating-frame GSS sweep uses $\alpha\in\{0.0,0.3,0.6,0.9,1.2,1.5\}$, 100 seeds, a 50-step planner horizon, and a 50-step rollout cap.
GSS uses action budget 64, state width 512, action mixture $0/1/31/32$ (one LQR guide action, 31 noisy-LQR actions, and 32 uniform actions), and an LQR guide and rollout policy.
Graph backups use exact-moment SNIS densities with rollout tail values, and the state proposal is the top-fraction transition-moment mixture described above, with top fraction $0.95$, fitted from transition moments and without proposal inflation, with direct successor samples~\eqref{eq:gss-backup-direct} at all depths.
The companion Julia DPW/VPW rotating-frame runs use the same DDI dynamics, 100 seeds, the same alpha grid, and per-action wall-clock budgets in seconds
\[
0.01,\ 0.0316227766,\ 0.1,\ 0.3162277660,\ 1.0.
\]
Tables~\ref{tab:ddi-budget-sweep}, \ref{tab:ddi-plan-time}, and~\ref{tab:ddi-early-termination} report the returns, planning times, and early terminations of every planner in Figure~\ref{fig:ddi_results}, with one row per planner and budget and one column per rotation $\alpha$.

\begin{table}[H]
\centering
\scriptsize
\setlength{\tabcolsep}{3pt}
\caption{Rotating DDI mean undiscounted return $\pm2$ standard errors over 100 seeds. For GSS, the budget is the action budget $K$ and state width $C$; for DPW and VPW, it is the per-decision time budget.}
\label{tab:ddi-budget-sweep}
\resizebox{\textwidth}{!}{\begin{tabular}{llcccccc}
\toprule
& & \multicolumn{6}{c}{$\alpha$} \\
\cmidrule(lr){3-8}
Planner & Budget & 0.0 & 0.3 & 0.6 & 0.9 & 1.2 & 1.5 \\
\midrule
LQR & -- & 48.939 $\pm$ 0.006 & 48.783 $\pm$ 0.010 & 47.666 $\pm$ 0.031 & 42.465 $\pm$ 0.128 & 11.307 $\pm$ 0.201 & 4.255 $\pm$ 0.054 \\
LQG & -- & 48.955 $\pm$ 0.007 & 48.757 $\pm$ 0.011 & 47.385 $\pm$ 0.039 & 40.162 $\pm$ 0.167 & 8.900 $\pm$ 0.138 & 4.071 $\pm$ 0.024 \\
\midrule
GSS & $K{=}64$, $C{=}512$ & 48.194 $\pm$ 0.035 & 48.085 $\pm$ 0.034 & 47.586 $\pm$ 0.045 & 46.796 $\pm$ 0.049 & 44.138 $\pm$ 0.839 & 37.258 $\pm$ 2.286 \\
\midrule
DPW & 0.01s & 48.172 $\pm$ 0.036 & 48.148 $\pm$ 0.037 & 47.362 $\pm$ 0.058 & 45.944 $\pm$ 0.076 & 42.266 $\pm$ 1.250 & 26.072 $\pm$ 3.492 \\
DPW & 0.0316s & 48.006 $\pm$ 0.043 & 47.885 $\pm$ 0.042 & 47.215 $\pm$ 0.048 & 46.090 $\pm$ 0.064 & 43.345 $\pm$ 0.767 & 32.929 $\pm$ 2.960 \\
DPW & 0.1s & 47.850 $\pm$ 0.045 & 47.750 $\pm$ 0.039 & 47.199 $\pm$ 0.051 & 46.118 $\pm$ 0.074 & 44.038 $\pm$ 0.081 & 34.781 $\pm$ 2.764 \\
DPW & 0.316s & 47.992 $\pm$ 0.039 & 47.823 $\pm$ 0.038 & 47.354 $\pm$ 0.041 & 46.360 $\pm$ 0.053 & 44.391 $\pm$ 0.084 & 38.134 $\pm$ 2.151 \\
DPW & 1.0s & 47.998 $\pm$ 0.032 & 47.920 $\pm$ 0.033 & 47.509 $\pm$ 0.035 & 46.546 $\pm$ 0.044 & 44.694 $\pm$ 0.074 & 39.242 $\pm$ 1.970 \\
\midrule
VPW & 0.01s & 48.113 $\pm$ 0.034 & 48.009 $\pm$ 0.031 & 47.333 $\pm$ 0.046 & 46.047 $\pm$ 0.082 & 42.892 $\pm$ 1.006 & 29.664 $\pm$ 3.290 \\
VPW & 0.0316s & 47.972 $\pm$ 0.031 & 47.905 $\pm$ 0.031 & 47.299 $\pm$ 0.043 & 46.241 $\pm$ 0.049 & 43.343 $\pm$ 1.055 & 34.332 $\pm$ 2.798 \\
VPW & 0.1s & 47.939 $\pm$ 0.029 & 47.836 $\pm$ 0.028 & 47.348 $\pm$ 0.038 & 46.390 $\pm$ 0.051 & 44.478 $\pm$ 0.071 & 36.641 $\pm$ 2.519 \\
VPW & 0.316s & 48.081 $\pm$ 0.025 & 47.956 $\pm$ 0.023 & 47.557 $\pm$ 0.030 & 46.649 $\pm$ 0.042 & 44.761 $\pm$ 0.067 & 36.387 $\pm$ 2.640 \\
VPW & 1.0s & 48.119 $\pm$ 0.023 & 48.064 $\pm$ 0.022 & 47.694 $\pm$ 0.027 & 46.825 $\pm$ 0.034 & 45.109 $\pm$ 0.055 & 38.262 $\pm$ 2.390 \\
\bottomrule
 \end{tabular}}
\end{table}

Table~\ref{tab:ddi-plan-time} reports the measured planning times.
The GSS time depends on $\alpha$ because the transition model is evaluated with one Euler step at $\alpha=0$ and with $32$ Euler substeps for $\alpha>0$.

\begin{table}[H]
\centering
\small
\setlength{\tabcolsep}{5pt}
\caption{Rotating DDI mean planning time per decision in seconds. Budgets are as in Table~\ref{tab:ddi-budget-sweep}.}
\label{tab:ddi-plan-time}
\begin{tabular}{llcccccc}
\toprule
& & \multicolumn{6}{c}{$\alpha$} \\
\cmidrule(lr){3-8}
Planner & Budget & 0.0 & 0.3 & 0.6 & 0.9 & 1.2 & 1.5 \\
\midrule
GSS & $K{=}64$, $C{=}512$ & 0.189 & 0.579 & 0.579 & 0.579 & 0.586 & 0.655 \\
\midrule
DPW & 0.01s & 0.041 & 0.026 & 0.027 & 0.027 & 0.028 & 0.032 \\
DPW & 0.0316s & 0.047 & 0.041 & 0.041 & 0.041 & 0.041 & 0.042 \\
DPW & 0.1s & 0.135 & 0.118 & 0.120 & 0.118 & 0.118 & 0.118 \\
DPW & 0.316s & 0.405 & 0.362 & 0.364 & 0.364 & 0.363 & 0.358 \\
DPW & 1.0s & 1.073 & 1.049 & 1.051 & 1.053 & 1.053 & 1.048 \\
\midrule
VPW & 0.01s & 0.033 & 0.027 & 0.027 & 0.027 & 0.027 & 0.029 \\
VPW & 0.0316s & 0.047 & 0.041 & 0.041 & 0.041 & 0.041 & 0.043 \\
VPW & 0.1s & 0.129 & 0.118 & 0.119 & 0.119 & 0.117 & 0.119 \\
VPW & 0.316s & 0.402 & 0.364 & 0.363 & 0.363 & 0.364 & 0.360 \\
VPW & 1.0s & 1.071 & 1.055 & 1.052 & 1.051 & 1.053 & 1.050 \\
\bottomrule
 \end{tabular}
\end{table}

Episodes end early only by leaving the state box.
No planner leaves it for $\alpha\le0.9$; at $\alpha\ge1.2$, LQR and LQG leave it in every episode.

\begin{table}[H]
\centering
\small
\setlength{\tabcolsep}{5pt}
\caption{Rotating DDI episodes, out of 100, that end early by leaving the state box. Budgets are as in Table~\ref{tab:ddi-budget-sweep}.}
\label{tab:ddi-early-termination}
\begin{tabular}{llcccccc}
\toprule
& & \multicolumn{6}{c}{$\alpha$} \\
\cmidrule(lr){3-8}
Planner & Budget & 0.0 & 0.3 & 0.6 & 0.9 & 1.2 & 1.5 \\
\midrule
LQR & -- & 0 & 0 & 0 & 0 & 100 & 100 \\
LQG & -- & 0 & 0 & 0 & 0 & 100 & 100 \\
\midrule
GSS & $K{=}64$, $C{=}512$ & 0 & 0 & 0 & 0 & 1 & 11 \\
\midrule
DPW & 0.01s & 0 & 0 & 0 & 0 & 3 & 40 \\
DPW & 0.0316s & 0 & 0 & 0 & 0 & 1 & 22 \\
DPW & 0.1s & 0 & 0 & 0 & 0 & 0 & 18 \\
DPW & 0.316s & 0 & 0 & 0 & 0 & 0 & 10 \\
DPW & 1.0s & 0 & 0 & 0 & 0 & 0 & 8 \\
\midrule
VPW & 0.01s & 0 & 0 & 0 & 0 & 2 & 31 \\
VPW & 0.0316s & 0 & 0 & 0 & 0 & 2 & 19 \\
VPW & 0.1s & 0 & 0 & 0 & 0 & 0 & 14 \\
VPW & 0.316s & 0 & 0 & 0 & 0 & 0 & 16 \\
VPW & 1.0s & 0 & 0 & 0 & 0 & 0 & 12 \\
\bottomrule
 \end{tabular}
\end{table}

\subsection{Lunar Lander}
\label{app:exp-lunar}

The Lunar Lander domain follows the benchmark used by \citet{Mern21aaai} and \citet{Lim21cdc}.
The state is $s=(x,z,\theta,v_x,v_z,\omega)$.
The action is $(F_x,T,\delta)$ with
\[
F_x\in[-5,5],\qquad T\in[0,15],\qquad \delta\in[-1,1].
\]
The parameters are mass $m=1$, inertia $I=10$, gravity $9.0$, and $\Delta t=0.4$.
The transition computes
\begin{equation}
\label{eq:lunar-transition}
\begin{aligned}
f_x &= \cos(\theta)F_x-\sin(\theta)T, &
f_z &= \cos(\theta)T+\sin(\theta)F_x,\\
\tau &= -\delta F_x, &
\dot\omega &= \tau/I,\\
v_x' &= v_x+\Delta t\, f_x/m+\epsilon_1\,0.1, &
v_z' &= v_z+\Delta t\,(f_z/m-9.0)+\epsilon_2\,0.1,\\
\omega' &= \omega+\Delta t\,\dot\omega+\epsilon_3\,0.01, &
x' &= x+\Delta t\,v_x,\\
z' &= z+\Delta t\,v_z, &
\theta' &= \theta+\Delta t\,\omega,
\end{aligned}
\end{equation}
where $\epsilon_1,\epsilon_2,\epsilon_3$ are independent standard normal variables.
The transition is therefore low-rank in the six-dimensional state, and GSS uses the single-point smoothed density.
The default initial mean is $(0,50,0,0,-10,0)$ with diagonal initialization standard deviations
\[
\sqrt{(0.1,0.1,0.01,0.1,0.1,0.01)}.
\]

Termination occurs when $\lvert x\rvert\ge 15$, $\lvert\theta\rvert\ge0.5$, or $z\le1$.
The reward is
\[
r(s,a,s')=
\begin{cases}
-1000, & \lvert x'\rvert\ge 15\ \text{or}\ \lvert\theta'\rvert\ge0.5,\\
100-\lvert x'\rvert-(v_z')^2, & z'\le1,\\
-1, & \text{otherwise}.
\end{cases}
\]
The discount is $\gamma=0.99$.

The Lunar Lander experiments use 100 seeds and a brake-near-ground policy for both guide and rollout.
Graph backups use single-point smoothed densities with rollout tail values and smoothing bandwidth $(0.02, 0.02, 0.005, 0.02, 0.02, 0.002)$.
The state proposal is the top-fraction transition-moment Gaussian described above, fitted to the sampled one-step successors, with top fraction $0.8$ and inflated by a factor $1.1$, and the backup includes the direct successor sample~\eqref{eq:gss-backup-direct}.
The planner horizon is 12, the rollout cap and closed-loop cap are 24, and the GSS action and state sweep uses the following configurations:
\[
\begin{array}{c|c|c}
\text{action budget} & \text{state width} & \text{primitive/rollout/noisy/uniform}\\
\hline
10 & 384 & 2/1/5/2\\
16 & 512 & 4/1/6/5\\
24 & 1024 & 6/1/8/9\\
40 & 1536 & 6/1/16/17\\
50 & 2048 & 6/1/21/22
\end{array}
\]
The Julia Lunar tree curves use the same brake-near-ground guide and rollout policy, 100 seeds, a 35-step closed-loop cap, which does not affect the comparison since every DPW and VPW episode lands within 24 steps, and per-action budgets
\[
0.0001,\ 0.000316228,\ 0.001,\ 0.00316228,\ 0.01,\ 0.0316228,\ 0.1,\ 0.316228,\ 1.0
\]
seconds.
DPW and VPW generate actions as described above: the guide-policy action first, then uniform (DPW) or VOO (VPW) samples.
Table~\ref{tab:lunar-planner-performance} reports the Lunar planner-performance rows corresponding to the plotted comparison.

\begin{table}[H]
\centering
\small
\setlength{\tabcolsep}{5pt}
\caption{Lunar Lander planner performance from Figure~\ref{fig:lunar_results} over 100 seeds. For GSS, the budget is the action budget $K$ and state width $C$; for DPW and VPW, it is the per-decision time budget. Time is the measured mean planning time per decision. Return entries are discounted mean episode return $\pm2$ standard errors. Landed counts the episodes that end by landing, and Samples is the number of states drawn per decision (Appendix~\ref{app:sample-counts}).}
\label{tab:lunar-planner-performance}
\begin{tabular}{llrcrr}
\toprule
Planner & Budget & Time (s) & Return & Landed & Samples \\
\midrule
GSS & $K{=}10$, $C{=}384$ & 0.0096 & 30.763 $\pm$ 40.288 & 96 & 0.77M \\
GSS & $K{=}16$, $C{=}512$ & 0.0114 & 65.407 $\pm$ 19.808 & 99 & 1.63M \\
GSS & $K{=}24$, $C{=}1024$ & 0.0287 & 75.570 $\pm$ 2.814 & 100 & 4.89M \\
GSS & $K{=}40$, $C{=}1536$ & 0.0529 & 78.641 $\pm$ 0.216 & 100 & 12.20M \\
GSS & $K{=}50$, $C{=}2048$ & 0.1065 & 78.618 $\pm$ 0.360 & 100 & 20.33M \\
DPW & 0.01s & 0.0196 & 73.582 $\pm$ 0.415 & 100 & $\le$0.11M \\
DPW & 0.0316s & 0.0418 & 73.219 $\pm$ 0.513 & 100 & $\le$0.31M \\
DPW & 0.1s & 0.1157 & 71.908 $\pm$ 0.632 & 100 & $\le$0.79M \\
DPW & 0.316s & 0.3347 & 70.125 $\pm$ 0.454 & 100 & $\le$1.79M \\
DPW & 1.0s & 1.0300 & 68.747 $\pm$ 0.502 & 100 & $\le$3.87M \\
VPW & 0.01s & 0.0189 & 74.402 $\pm$ 0.707 & 100 & $\le$0.10M \\
VPW & 0.0316s & 0.0414 & 73.795 $\pm$ 0.592 & 100 & $\le$0.28M \\
VPW & 0.1s & 0.1159 & 73.391 $\pm$ 0.322 & 100 & $\le$0.71M \\
VPW & 0.316s & 0.3385 & 71.161 $\pm$ 0.658 & 100 & $\le$1.68M \\
VPW & 1.0s & 1.0315 & 69.835 $\pm$ 0.537 & 100 & $\le$3.66M \\
\bottomrule
 \end{tabular}
\end{table}

\paragraph{Tree-planner returns across budgets.}
The DPW and VPW returns in Table~\ref{tab:lunar-planner-performance} decrease as the budget grows beyond $0.01$ seconds, while every DPW and VPW episode lands, so the decline reflects landing quality, namely touchdown position, touchdown speed, and episode length, rather than crashes.
Progressive-widening planners do not guarantee monotone improvement at finite budgets, since a larger budget changes the value estimates and therefore the recommended action.
The rollout-policy action is the first action expanded at every node, and we hypothesize that as the budget grows, the search departs further from the rollout policy before its tree is large enough to evaluate the alternative actions reliably.

\subsection{Reacher}

The Reacher benchmark is implemented in JAX~\citep{Bradbury18github} using the MJX interface to MuJoCo~\citep{Todorov12iros}, with a Gymnasium Reacher-v5-style model and observation/action convention~\citep{Towers24arxiv}.
The MJX model has $n_q=4$, $n_v=4$, and $n_u=2$.
The planner state is the 10D observation
\[
s=\bigl(\cos q_0,\cos q_1,\sin q_0,\sin q_1,g_x,g_y,\dot q_0,\dot q_1,d_x,d_y\bigr),
\]
where $g=(g_x,g_y)$ is the target coordinate and $d$ is the fingertip-target displacement.
Actions are 2D torques clipped to the XML actuator control range $[-1,1]^2$.
The transition converts the 10D observation back to MJX state variables, applies the clipped control for frame skip two with $\Delta t=0.02$, and maps the resulting MJX data back to the same 10D observation.
There is no transition noise and no early terminal condition.
Initial states sample the two arm angles uniformly from $[-0.1,0.1]$, arm angular velocities from $[-0.005,0.005]$, and the target uniformly in a disk of radius $0.2$.
The reward is
\[
r(s,a,s')=-\lVert d'\rVert_2-\lVert a\rVert_2^2,
\]
and the discount is $\gamma=0.99$.

The default controller used for the PD baseline, guide policy, and rollout policy is a clipped task-space PD-style controller.
From the observation, it recovers $q=(q_0,q_1)$ by applying $\operatorname{atan2}$ to the sine and cosine coordinates, takes $\dot q=(\dot q_0,\dot q_1)$ from the velocity entries, and uses $d=(d_x,d_y)$ as the fingertip-target displacement.
With link lengths $\ell_1=\ell_2=0.1$, let
\[
J(q)=
\begin{pmatrix}
-\ell_1\sin q_0-\ell_2\sin(q_0+q_1) & -\ell_2\sin(q_0+q_1)\\
\ell_1\cos q_0+\ell_2\cos(q_0+q_1) & \ell_2\cos(q_0+q_1)
\end{pmatrix}
\]
be the planar fingertip Jacobian.
The implemented controller is
\[
a_{\mathrm{PD}}(s)=\operatorname{clip}_{[-1,1]^2}
\bigl(-12\,J(q)^\top d-0.35\,\dot q\bigr).
\]
Thus the proportional term acts on fingertip-target displacement through the Jacobian transpose, and the derivative term damps the two arm angular velocities.

Because the Reacher transition is deterministic in the ambient 10D observation representation, the domain only permits the single-point smoothed density with positive bandwidth.
The Reacher GSS sweeps use dense SNIS backups with rollout tail values: graph leaves are valued by a single rollout under the default rollout policy, capped by the 50-step rollout/episode horizon.
The state proposal is an unpruned transition-moment mixture fitted from transition moments, without proposal inflation; rollout-based proposal pruning and direct successor samples are not used.
The action guide is the default controller, and each action set contains one guide action plus noisy-guide and uniform samples.
The bandwidth vector is
\[
(0.05,0.05,0.05,0.05,0.02,0.02,0.10,0.10,0.03,0.03).
\]
The closed-loop cap and rollout cap are 50 steps, and all Reacher GSS budget-scaling experiments use planner horizon 20 and 20 seeds.
The reported sweep varies state width and action budget over
\begin{gather*}
(256,24),\ (512,24),\ (512,32),\ (512,48),\ (768,32),\\ (768,48),\ (1024,48),\ (1280,48),\ (1024,64),\ (1536,48), \\
(1280,64),\ (1792,48),\ (1536,64),\ (2048,48),\ (2048,64),
\end{gather*}
where each pair is (state width, action budget).
The action mixtures, reported as rollout/noisy/uniform, are $1/3/20$ for action budget 24, $1/4/27$ for action budget 32, $1/6/41$ for action budget 48, and $1/8/55$ for action budget 64.

Table~\ref{tab:reacher-budget-sweep} reports the Reacher budget-scaling results used in Figure~\ref{fig:reacher_results}, together with the PD controller, which makes no planner call.

\begin{table}[H]
\centering
\small
\setlength{\tabcolsep}{5pt}
\caption{Reacher GSS budget-scaling rows and the PD controller. Time is mean planning time per decision in seconds. Return entries are undiscounted 50-step return, and final distance is the fingertip-target distance at the end of the episode, both $\pm2$ standard errors.}
\label{tab:reacher-budget-sweep}
\begin{tabular}{rrrcc}
\toprule
State width & Action budget & Time (s) & Return & Final distance \\
\midrule
PD controller & -- & -- & -9.358 $\pm$ 2.464 & 0.159 $\pm$ 0.047 \\
\midrule
256 & 24 & 0.40 & -8.006 $\pm$ 1.724 & 0.082 $\pm$ 0.019 \\
512 & 24 & 0.54 & -7.552 $\pm$ 1.554 & 0.081 $\pm$ 0.021 \\
512 & 32 & 0.60 & -7.549 $\pm$ 1.358 & 0.077 $\pm$ 0.015 \\
512 & 48 & 0.80 & -7.097 $\pm$ 1.276 & 0.067 $\pm$ 0.016 \\
768 & 32 & 0.83 & -6.982 $\pm$ 1.279 & 0.063 $\pm$ 0.017 \\
768 & 48 & 1.21 & -6.785 $\pm$ 1.156 & 0.058 $\pm$ 0.013 \\
1024 & 48 & 1.70 & -6.631 $\pm$ 0.720 & 0.060 $\pm$ 0.011 \\
1280 & 48 & 2.27 & -6.308 $\pm$ 0.907 & 0.061 $\pm$ 0.020 \\
1024 & 64 & 2.40 & -6.240 $\pm$ 0.919 & 0.055 $\pm$ 0.016 \\
1536 & 48 & 2.93 & -6.302 $\pm$ 0.924 & 0.054 $\pm$ 0.017 \\
1280 & 64 & 3.22 & -6.384 $\pm$ 1.025 & 0.052 $\pm$ 0.016 \\
1792 & 48 & 3.49 & -6.267 $\pm$ 0.860 & 0.058 $\pm$ 0.019 \\
1536 & 64 & 4.10 & -6.280 $\pm$ 0.903 & 0.054 $\pm$ 0.017 \\
2048 & 48 & 4.23 & -6.123 $\pm$ 0.796 & 0.056 $\pm$ 0.019 \\
2048 & 64 & 5.69 & -6.144 $\pm$ 0.855 & 0.052 $\pm$ 0.014 \\
\bottomrule
 \end{tabular}
\end{table}
 
\clearpage
\section*{NeurIPS Paper Checklist}

\begin{enumerate}

\item {\bf Claims}
    \item[] Question: Do the main claims made in the abstract and introduction accurately reflect the paper's contributions and scope?
    \item[] Answer: \answerYes{} \item[] Justification: All of our reported results are either proven theoretically, or based on specific simulative expeirments which we present in the paper and detail in the Appendix.

\item {\bf Limitations}
    \item[] Question: Does the paper discuss the limitations of the work performed by the authors?
    \item[] Answer: \answerYes{} \item[] Justification: Each theoretical result states its assumptions explicitly, and the Conclusion discusses the limitations of GSS and the scope of the guarantees, and Section~\ref{sec:experiments} describes how the implementation departs from the analyzed algorithm. Appendix~\ref{app:sample-counts} reports sample counts alongside the CPU/GPU time comparisons.

\item {\bf Theory assumptions and proofs}
    \item[] Question: For each theoretical result, does the paper provide the full set of assumptions and a complete (and correct) proof?
    \item[] Answer: \answerYes{} \item[] Justification: We rigorously state all assumptions in the paper and provide full theorem statements in the main paper. All proofs are provided in the appendix.

    \item {\bf Experimental result reproducibility}
    \item[] Question: Does the paper fully disclose all the information needed to reproduce the main experimental results of the paper to the extent that it affects the main claims and/or conclusions of the paper (regardless of whether the code and data are provided or not)?
    \item[] Answer: \answerYes{} \item[] Justification: We detail the setup and configurations of all experiments in the paper Appendix.

\item {\bf Open access to data and code}
    \item[] Question: Does the paper provide open access to the data and code, with sufficient instructions to faithfully reproduce the main experimental results, as described in supplemental material?
    \item[] Answer: \answerYes{} \item[] Justification: Our code and experimental data are publicly available at \url{https://github.com/IdanLevYehudi/GraphSparseSampling} under the MIT license. The repository documents the software environment, the configurations and commands that run GSS and the DPW/VPW baselines, and the packaged result data from which the paper figures are regenerated.

\item {\bf Experimental setting/details}
    \item[] Question: Does the paper specify all the training and test details (e.g., data splits, hyperparameters, how they were chosen, type of optimizer) necessary to understand the results?
    \item[] Answer: \answerYes{} \item[] Justification: We explain the principal algorithmic choices in the main paper body, and write the full algorithmic details in the Appendix.

\item {\bf Experiment statistical significance}
    \item[] Question: Does the paper report error bars suitably and correctly defined or other appropriate information about the statistical significance of the experiments?
    \item[] Answer: \answerYes{} \item[] Justification: We report error bars for all experiments and define them as $\pm 2$ standard-error intervals in the figure and table captions.

\item {\bf Experiments compute resources}
    \item[] Question: For each experiment, does the paper provide sufficient information on the computer resources (type of compute workers, memory, time of execution) needed to reproduce the experiments?
    \item[] Answer: \answerYes{} \item[] Justification: The compute resources used for our experiments are detailed in the Appendix.
    
\item {\bf Code of ethics}
    \item[] Question: Does the research conducted in the paper conform, in every respect, with the NeurIPS Code of Ethics \url{https://neurips.cc/public/EthicsGuidelines}?
    \item[] Answer: \answerYes{} \item[] Justification: Reviewed the NeurIPS Code of Ethics and confirm that our research conforms to it.

\item {\bf Broader impacts}
    \item[] Question: Does the paper discuss both potential positive societal impacts and negative societal impacts of the work performed?
    \item[] Answer: \answerNA{} \item[] Justification: The work is theoretical and does not have direct societal impacts.
    
\item {\bf Safeguards}
    \item[] Question: Does the paper describe safeguards that have been put in place for responsible release of data or models that have a high risk for misuse (e.g., pre-trained language models, image generators, or scraped datasets)?
    \item[] Answer: \answerNA{} \item[] Justification: The paper does not release any data or models.

\item {\bf Licenses for existing assets}
    \item[] Question: Are the creators or original owners of assets (e.g., code, data, models), used in the paper, properly credited and are the license and terms of use explicitly mentioned and properly respected?
    \item[] Answer: \answerYes{} \item[] Justification: We cite all the original papers for the code packages and datasets we use, and we respect the licenses of all of them. No new dataset asset is introduced.

\item {\bf New assets}
    \item[] Question: Are new assets introduced in the paper well documented and is the documentation provided alongside the assets?
    \item[] Answer: \answerYes{} \item[] Justification: We release our code and experimental data at \url{https://github.com/IdanLevYehudi/GraphSparseSampling} under the MIT license, documented in the repository.

\item {\bf Crowdsourcing and research with human subjects}
    \item[] Question: For crowdsourcing experiments and research with human subjects, does the paper include the full text of instructions given to participants and screenshots, if applicable, as well as details about compensation (if any)? 
    \item[] Answer: \answerNA{} \item[] Justification: This research does not involve crowdsourcing nor research with human subjects.

\item {\bf Institutional review board (IRB) approvals or equivalent for research with human subjects}
    \item[] Question: Does the paper describe potential risks incurred by study participants, whether such risks were disclosed to the subjects, and whether Institutional Review Board (IRB) approvals (or an equivalent approval/review based on the requirements of your country or institution) were obtained?
    \item[] Answer: \answerNA{} \item[] Justification: This paper does not involve research with human subjects.

\item {\bf Declaration of LLM usage}
    \item[] Question: Does the paper describe the usage of LLMs if it is an important, original, or non-standard component of the core methods in this research? Note that if the LLM is used only for writing, editing, or formatting purposes and does \emph{not} impact the core methodology, scientific rigor, or originality of the research, declaration is not required.
\item[] Answer: \answerNA{} \item[] Justification: The paper's method does not involve LLMs as any important, original, or non-standard components.

\end{enumerate}
 
\end{document}